\def\year{2018}\relax

\documentclass[letterpaper]{article}
\usepackage{aaai18}
\usepackage{times}

\usepackage{graphicx}
\usepackage{times}

\usepackage{mathrsfs}
\usepackage{dsfont}
\usepackage{amsmath}
\usepackage{amsthm}
\usepackage{amssymb}
\usepackage[norelsize,ruled,vlined,linesnumbered]{algorithm2e}

\usepackage{hyperref}

\long\def\comment#1{}

\newtheorem{definition}{Definition}
\newtheorem{theorem}{Theorem}

\newtheorem{proposition}{Proposition}
\newtheorem{lemma}{Lemma}
\newtheorem{corollary}[theorem]{Corollary}
\newtheorem{examp}{Example}
\newenvironment{theorem-appendix}{\vspace*{0.25cm}\textit{Theorem} }{} 
\newenvironment{proofsketch}{\noindent \textit{Proof}\textit{ (Sketch)}. }{$\Box$\\ }
\newenvironment{proposition-appendix}{\vspace*{0.25cm}\textit{Proposition} }{} 
\newenvironment{example}{\begin{examp}\rm}{\end{examp}}
\newcommand{\la}{\leftarrow}

\newcommand{\ra}{\rightarrow}

\newcommand{\LA}{\langle}
\newcommand{\RA}{\rangle}
\newcommand{\SC}{\textsc}
\newcommand{\vect}{\mathbf}
\newcommand{\REST}{\hspace{-0.05in}\restriction}

\title{Loop Restricted Existential Rules and First-order Rewritability for Query Answering}
\author{Vernon Asuncion$^{\,a}$, Yan Zhang$^{\,a,\,c}$, Heng Zhang$^{\,b}$, Yun Bai$^{\,a}$ and Weisheng Si$^{\,a}$\\
	$^{a\,}$School of Computing Engineering and Mathematics\\
	Western Sydney University, Australia\\
	$^{b\,}$School of Computer Software, Tianjin University, China\\
	$^{c\,}$Huazhong University of Science \& Technology, China}  

\begin{document}

\maketitle

\begin{abstract}
  In ontology-based data access (OBDA), the classical 
  database is enhanced with an ontology in the form
  of logical assertions generating new intensional
  knowledge. A powerful form of such logical
  assertions is the tuple-generating
  dependencies (TGDs), also called existential rules, where
Horn rules are extended by allowing 
existential quantifiers to appear in the rule heads. 
In this paper we introduce a new language called {\em loop restricted} (LR) TGDs (existential rules), which are
TGDs with certain restrictions on the loops embedded in the underlying rule set.
We study the complexity of this new language. We show that the conjunctive query answering 
(CQA) under the LR TGDs is decidable. In particular, 
we prove that 
this language satisfies the so-called bounded derivation-depth property (BDDP), which
implies that the CQA is first-order rewritable, and
its data complexity is in \textsc{AC}$^{0}$.
We also prove that the combined complexity of the CQA 
is \textsc{ExpTime} complete, while the language membership 
is \textsc{Pspace} complete.
Then we extend the LR TGDs language to the generalised loop restricted (GLR) TGDs language, and prove that this class of 
TGDs still remains to be first-order rewritable and properly contains most of other first-order rewritable TGDs classes discovered in
the literature so far.

\comment{

  Horn rules extended with the possibility of existential quantification 
  in the head. On the other hand, allowing existential
  quantification to appear in the heads makes query answering under TGDs 
  in general undecidable, even under finite databases. 
  Thus, an important research direction is to identify
  classes of TGDs where the query answering problem is decidable.
  In this paper, we introduce a novel class of TGDs 
  called loop restricted (LR) TGDs which, to the 
  best of our knowledge, is a new useful decidable class not yet discovered in 
  the literatures. 
}

\end{abstract}

\section{Introduction}


In ontology-based data access (OBDA), a
database is enhanced with an ontology in the form
of logical assertions generating new intensional
knowledge, e.g.,  \cite{baader:jair16,bien:ijcai16,eiter:kr16,kon:2013,niko:2017}.
\comment{
In ontology-enhanced DBMS, an extensional
database $D$ (sometimes called ABox) is combined with an ontology
that specifies rules and constraints that derive new intensional information from
the extensional input. 

this is is vastly different from earlier
notions of querying databases in that queries are answered
against the logical theory $D$ $\cup$ $\Sigma$ rather than
just the database $D$. As such, for a given 
\textit{conjunctive query (CQ)} $q$ of the form
$q(\vect{X})$ $\la$ $\mathsf{body}(\vect{X},\vect{Y})$ such that
$\vect{X}$ are the ``output" variables, $q$'s answer 
are those tuples $\vect{c}$ of constants
such that $D$ $\cup$ $\Sigma$ $\models$ $\exists\vect{Y}\mathsf{body}(\vect{c},\vect{Y})$. 
}
A powerful form of such logical
assertions is the {\em tuple-generating
dependencies (TGDs)}, also called {\em existential rules}. 
Generally speaking, TGDs are 
Horn rules extended by allowing the occurrence of existential quantification 
in the rule head. With this extension, it is able to reason about the 
existence of new or missing objects that are not represented in the 
underlying database \cite{BagetLMS11,Patel-SchneiderH07}. 

Under the language of TGDs, queries are answered against an ontology represented by a set 
of TGDs and an input database. In particular, given a database instance $D$, a finite set 
$\Sigma$ of TGDs, and a query $q$, we want to decide whether $D\cup\Sigma\models q$. 
\comment{
Using the chase procedure, 
which is originally an algorithmic approach in database for 
dependency implication and query containment checking \cite{DeutschNR08,FaginKMP05,JohnsonK84,MaierMS79,ZhangZY15},
this problem is equivalent to deciding whether $\mathsf{chase}(D,\Sigma)\models q$.
}
However, 
this problem is undecidable generally, due to the
potential cyclic applications of TGDs in $\Sigma$ \cite{DeutschNR08}.

In recent years, considerable research has been carried out to identify various
 expressive decidable classes of TGDs. 
So far 
several primary such classes have been discovered:  
\textit{weakly-acyclic} class \cite{FaginKMP05};
\textit{guarded} class \cite{BagetLMS11,CaliGK08,CaliGL12};
\textit{sticky sets} class \cite{CaliGP12}; and \textit{Shy programs} class \cite{LeoneMTV12}.
By extending and combining these aforementioned classes, 
more decidable classes can be derived, 
%
such as 
\textit{glut-guardedness} (weak-acyclicity + guardedness) \cite{KrotzschR11};
\textit{weak-stickiness} (weak-acyclicity + stickiness) \cite{CaliGP12}; \textit{model-faithful acyclicity} (MFA) \cite{g13}; 
and \textit{tameness} (guardedness + stickiness) \cite{GottlobMP13}.

\comment{
Unfortunately, there are still real life scenarios that are qute intuitive but not expressible by any of the existing 
decidable existential rule languages, as illustrated by the following example.

weakly recursive TGDs \cite{wr-2012}.
}

Among all these decidable classes, some are of special interests for OBDA, i.e., 
the classes of first-order rewritable TGDs, where
conjunctive query answering can be reduced to the
evaluation of a first-order query over the database. As such, traditional database query techniques may be
used for developing efficient query answering systems in OBDA, as demonstrated in Description Logics 
\cite{hansen:ijcai15,Kam:2014}.
 So far, several useful first-order rewritable classes of TGDs have been
discovered: {\em acyclic} TGDs, {\em aGRD} TGDs, 
{\em linear} and {\em multi-linear} TGDs, {\em sticky} and {\em sticky-join} TGDs, while
multi-linear and sticky-join TGDs generalise linear TGDs and sticky TGDs, respectively 
\cite{CaliGL12,CaliGP12}.

Civili and Rosati 
\cite{wr-2012}
further identified another first-order rewritable class called 
{\em weakly recursive} TGDs, and showed that by restricting to simple TGDs,
weakly recursive class contains all other first-order rewritable classes. 
%

Unfortunately, there are still real life scenarios that are simple and
intuitive but not syntactically recognisable by any of the existing 
first-order rewritable TGDs classes, as illustrated by the following example.

\begin{example}
Consider a university research domain, where we have the following ontology $\Sigma_{\mathsf{Research}}$ to represent its knowledge rules:
%
%
%
We have the following ontology $\Sigma_{\mathsf{Research}}$ to represent this domain:
\begin{quote}
$\sigma_1: \mathsf{resAdvisor}(X,W)\ra \mathsf{seniorStaff}(X)$, 

\vspace*{.02in}

$\sigma_2: \mathsf{seniorStaff}(X),\mathsf{advCommittee}(X,Y)$, \\
\hspace*{.2in}
$\mathsf{projDept}(X,Y) \ra$
$\exists W \mathsf{resAdvisor}(X,W)$.\\  
\vspace*{.02in}
\hspace*{-0.15cm}$\sigma_3: \mathsf{resStudent}(W) \ra \exists XY \mathsf{resAdvisor}(X,W)$,\\
\hspace*{.4in} $\mathsf{enrolDept}(W,Y),  \mathsf{projDept}(W,Y,Y)$, \\
\hspace*{.4in}  $\mathsf{projDept}(X,Y,Y)$.
\end{quote}
%
$\sigma_1$ says that if $X$ is a research advisor of someone,
then $X$ must be
a senior staff; 
$\sigma_2$ states that
if $X$ is a senior staff and a member of the department $Y$'s advisory committee, and 
$X$ also
undertakes a project registered in department $Y$, then $X$ must be a research advisor of 
someone\footnote{In general, $\mathsf{projDept}(X,Y,Z)$ means that staff $X$ from department $Y$ undertakes
a project registered in department $Z$.};
and
$\sigma_3$ indicates that a research student $W$ must have an advisor and should undertake 
the project together with the advisor from the same department, while the project has to be also registered in 
the this department.

Through a careful examination, it is not difficult to see that $\Sigma_{\mathsf{Research}}$ 
is not recognizable under the syntactic conditions of
all currently known first-order rewritable TGDs classes.
On the other hand, by unfolding the derivations on atoms $\mathsf{seniorStaff}(X)$ and $\mathsf{resAdvisor}(X,W)$ from
$\Sigma_{\mathsf{Research}}$, it turns out that their derivations are always bounded by a fixed length independent from
any input database. That is, the underlying $\Sigma_{\mathsf{Research}}$ satisfies the so-called BDDP property, from
which we know that the query answering under $\Sigma_{\mathsf{Research}}$ is not only 
decidable, but also first-order rewritable \cite{CaliGL12}.
\comment{we also note  that 
since the heads of $\sigma_1$ and $\sigma_2$ share the same universally quantified variable $X$, which is 
also exactly the same as the common variable of the recursive atoms in the bodies of $\sigma_1$ and $\sigma_2$, i.e., 
$\mathsf{seniorStaff}(X)$ and $\mathsf{seniorStaff}(X)$, respectively, 
this leads that the derivations for both atoms $\mathsf{seniorStaff}(X)$ and $\mathsf{resAdvisor}(X,W)$
will be always finite for all input databases, and 
furthermore, the derivation depths of these atoms are bounded
by a fixed number independent from the input database.
This would conclude that the query answering problem of $\Sigma_{\mathsf{Research}}$ is decidable, 
and also first-order rewritable.}
$\Box$
\label{ex1}
\end{example}

Main contributions of this paper are summarised here:
\begin{enumerate}
\item We define notations of {\em derivation paths} and {\em derivation 
trees} for query answering over TGDs (existential rules), and provide
a precise characterisation for the traditional TGDs chase procedure through the corresponding derivation tree
(Section 3).
\item Based on the concept of derivation paths, we 
introduce a new class called {\em loop restricted} (LR) TGDs,
which are TGDs with certain restrictions on the loops embedded in the underlying
rule set (Section 4).
\item Under our derivation tree framework, 
we show that the conjunctive query answering (CQA) under LR TGDs satisfies a property called 
{bounded derivation tree depth property (BDTDP)}. We further prove that BDTDP implies
the well-known bounded derivation-depth property (BDDP). This result implies that 
conjunctive query answering under LR TGDs is not only decidable but also first-order rewritable
(Section 4).
\item We further extend LR TGDs to {\em generalised loop restricted} (GLR) TGDs, and prove that the class of GLR TGDs is also first-order rewritable and contains most of other first-order rewritable TGD classes discovered in the literature so far
(Sections 5 and 6).
\end{enumerate}

\section{Preliminaries}


In this section, we introduce necessary notions and definitions we will need through out this paper.

\vspace*{.1in}

\noindent{\bf Databases and queries}. 
We define the following pairwise disjoint (countably infinite) sets of symbols:
a set $\Gamma$ of {\em constants}, which constitute the domain of
databases, a set $\Gamma_{N}$ of {\em labeled nulls} that
will be used as ``fresh" Skolem terms as placeholders for unknown
values, and a set $\Gamma_{V}$ of regular {\em variables}. For convenience, we usually use
$a, b, c, \cdots$ to denote constants, $\mathsf{n, n', n''} \cdots$ to denote nulls, and 
$X,Y, Z, \cdots$ to denote 
variables\footnote{Possibly these constants, nulls and variable are subscripted with indexes.}.
Note that different nulls may also represent the same value. 
We assume a lexicographic order on
$\Gamma\cup\Gamma_{N}$, with every symbol in $\Gamma_{N}$ following all symbols in $\Gamma$. We
use $\mathbf X$ to denote a sequence of variables $X_1,\cdots,X_n$, where $n\ge 0$. Sometimes, we also
represent such $\mathbf{X}$ as a $n$-ary tuple of variables $(X_1,\cdots,X_n)$.
A similar notion also applies to nulls.

A {\em relational schema} $\mathcal R$ is a finite set of {\em relation symbols} 
(or {\em predicates}). A {\em term} is a constant, null or variable. An {\em atom} has the form
$p(t_1,\cdots,t_n)$, where $p$ is an $n$-ary predicate, and $t_1,\dots, t_n$ are terms.
We denote by $|p|$ and $\mathsf{dom}(p)$ as $p$'s arity and the set of all its
terms respectively. The latter notion is naturally extended to sets of atoms and conjunctions of atoms. A
conjunction of atoms is often identified with the set of all its atoms. 

A {\em database} $D$ for a relational schema ${\mathcal R}$ is a finite set of atoms with
predicates from ${\mathcal R}$ and constants from $\Gamma$. 
That is, $\mathsf{dom}(D)\subseteq \Gamma$.  We also use $\mathsf{pred}(D)$ to denote the 
set of all predicates occurring in $D$.
An {\em instance} $I$ for a relational schema ${\mathcal R}$ is a (possibly infinite) set 
of atoms with predicates
from ${\mathcal R}$ and terms from $\Gamma\cup\Gamma_{N}$. Clearly, each database $D$ for 
${\mathcal R}$
may be viewed as a special form of instance, and further, it can be extended to an instance 
$I$ such that $D\subseteq I$ and $\mathsf{pred}(I)={\mathcal R}$.

A {\em homomophism} from a set of atoms $\mathbf{A}$ to a set of atoms $\mathbf{A'}$ 
is a mapping $h:$ $\Gamma$ $\cup$ $\Gamma_{N}$ $\cup$ $\Gamma_{V}$ $\rightarrow$ $\Gamma$
$\cup$ $\Gamma_{N}$ $\cup$ $\Gamma_{V}$,
such that (i) if $t\in \Gamma$, then $h(t)=t$; (ii) if $t\in \Gamma_N$, then $h(t)\in \Gamma\cup\Gamma_{N}$; 
and
(iii) if $p(t_1,\cdots,t_n)\in \mathbf{A}$, then $p(h(t_1),\cdots,h(t_n))\in \mathbf{A'}$. 
Let $\mathbf{T}$ be the set of all terms occurring in $\mathbf{A}$. 
The {\em restriction} $h'$ of $h$ to $\mathbf{S}\subseteq \mathbf{T}$, denoted as
$h'=h|_{\mathbf{S}}$, is simply the subset of $h:$ $h'=\{t\rightarrow h(t)\mid t\in \mathbf{S}\}$.
Here we also call $h$ is an {\em extension} of $h'$ to $\mathbf{T}$.

A {\em conjunctive query (CQ)} $q$ of arity $n$ over a schema ${\mathcal R}$ has the form
$p(\mathbf{X})\leftarrow \exists \mathbf{Y} \varphi(\mathbf{X,Y})$, where
$\varphi(\mathbf{X,Y})$ is a conjunction of atoms with the variables
$\mathbf{X}$ and $\mathbf{Y}$ from $\Gamma_{V}$ and constants from $\Gamma$, but without nulls, and $p$ is an $n$-ary predicate not occurring
in $\mathcal R$.
We allow $\varphi(\mathbf{X,Y})$ to contain equalities but no inequalities. 
When $\varphi(\mathbf{X},\mathbf{Y})$ is just a single atom, then we say that the CQ $q$ is 
\textit{atomic}. 
A {\em Boolean Conjunctive Query (BCQ)}
over ${\mathcal R}$ is a CQ of zero arity. In this case,
we can simply write a BCQ $q$ as $\exists \mathbf{Y} \varphi(\mathbf{Y})$. 
\comment{
For simplicity, we usually write a BCQ as the set of atoms with corresponding constants and variables as the
arguments, and omitting the quantifiers.
} 
A CQ answering problem, or called CQA problem, defined to be  
the {\em answer} to a CQ $q$ with $n$ arity over an instance $I$, denoted as $q(I)$,
is the set of all $n$-tuples $\mathbf{t}\in \Gamma^{n}$ for which there exists a homomorphism 
$h: \mathbf{X}\cup\mathbf{Y}\rightarrow \Gamma\cup \Gamma_{V}$ such that
$h(\varphi(\mathbf{X},\mathbf{Y}))\subseteq I$ and $h(\mathbf{X})=\mathbf{t}$. 
The answer to a BCQ is {\em positive} over $I$, denoted as $I\models q$, if 
$\langle\rangle \in q(I)$.

\vspace*{.1in}
\noindent
{\bf TGDs and conjunctive query answering (CQA)}.
A tuple-generating dependency (TGD) $\sigma$, also called {\em existential rule},  over a schema $\mathcal R$ is a first-order 
formula of the form
\begin{eqnarray}\label{lb1}
  \sigma:  \forall \mathbf{XY} \varphi(\mathbf{X}, \mathbf{Y})\rightarrow \exists\mathbf{Z}\psi(\mathbf{X},\mathbf{Z}),    
\end{eqnarray}
where $\mathbf{X}$ $\cup$ $\mathbf{Y}$ $\cup$ $\mathbf{Z}$ $\subset$ $\Gamma$ $\cup$ $\Gamma_{V}$, 
$\varphi$ and $\psi$ are conjunctions of atoms over $\mathcal R$.
When there is no confusion, we usually omit the universal quantifiers  from (\ref{lb1}). 
In this case, we also 
use $\mathsf{head}(\sigma)$ and $\mathsf{body}(\sigma)$ to denote formulas 
$\exists\mathbf{Z}\psi(\mathbf{X},\mathbf{Z})$ and 
$\varphi(\mathbf{X}, \mathbf{Y})$ respectively.

Let $I$ be an instance over $\mathcal R$. We say that $\sigma$ {\em is satisfied} in $I$, 
denoted as
$I$ $\models$ $\sigma$, if whenever there is a homomorphism $h$ such that 
$h(\varphi(\mathbf{X},\mathbf{Y}))$ $\subseteq$ $I$, then there 
exists an extension $h'$ of $h|_{\mathbf{X}}$ such that 
$h'(\psi(\mathbf{X},\mathbf{Z}))$ $\subseteq$ $I$.

Given a database $D$, a (finite) set $\Sigma$ of TGDs and a CQ $q$ of arity $n$ over schema $\mathcal R$. 
The {\em models} of $D$ with respect to $\Sigma$, denoted as
$\mathsf{mod}(D,\Sigma)$, is the set of all instances $I$ such that $I\supseteq D$ and $I\models\Sigma$.
Then a CQ answering problem, or called CQA problem, denoted as 
$\langle {\mathcal R}, D, \Sigma, q\rangle$, 
is described as follows:
the {\em answer } to $q$ with respect to $D$ and $\Sigma$, 
denoted as $\mathsf{ans}(q,D,\Sigma)$, is the set of all tuples: $\{\mathbf{t}\mid \mathbf{t}\in q(I)$, 
for each $I\in\mathsf{mod}(D,\Sigma)\}$. 
When $q$ is a BCQ, the answer to $q$ is called {\em positive} if 
$\langle\rangle \in \mathsf{ans}(q,D,\Sigma)$.
\comment{
Then the {\em (Boolean) conjunctive query answering problem}, or BCQA/CQA for short, is defined as follows: 
given a CQ $q$ with arity $n$ over schema $\mathcal R$, a database $D$ for $\mathcal R$, a set $\Sigma$ of 
existential rules over $\mathcal R$, and an tuple $\mathbf{t}\in \Gamma^{n}$, 
decide whether $\mathbf{t}\in \mathsf{ans}(q,D,\Sigma)$.
}
It is well known that the CQA problem and the problem of CQ containment under TGDs
 are \textsc{LogSpace}-equivalent, 
and hence, in the rest of this paper, we will only focus on
the BCQA problem, because all complexity results can be carried over to other problems \cite{CaliGP12}.


\vspace*{.1in}
\noindent
{\bf The chase algorithm}.
\comment{
The {\em chase} is a fundamental algorithm used in databases for checking implication of dependencies and query containment, as well as for computing solutions 
It has been showed that the chase can be also used to compute the answers to CQs under TGDs rules.
Now we present the TGD chase procedure\footnote{Note that the chase procedure works on an instance through TDG chase rule with two forms 
    named {\em oblivious} and
    {\em restricted} TGD chase rule respectively. Here we will focus on the first form. The key difference between these two forms of TGD chase rule is 
    described in \cite{CaliGP12}.}.
}
Consider an instance $I$ and a TGD
 $\sigma$ of the form (\ref{lb1}). We say that $\sigma$ is {\em applicable} to $I$ if there exists
a homomorphism $h$ such that
$h(\varphi(\mathbf{X},\mathbf{Y}))\subseteq I$. 
The {\em result} of applying $\sigma$ to $I$ is an instance 
$I'=I\cup h'(\psi(\mathbf{X},\mathbf{Z}))$, where $h'$ is an extension of $h|_{\mathbf{X}}$ such 
that for each $Z\in \mathbf{Z}$, $h'(Z)$ is a ``fresh" labeled null of $\Gamma_{N}$ not occurring in $I$, and following 
lexicographically all those in $I$. Then the oblivious TGD chase algorithm for a database $D$ and 
a set $\Sigma$ of TGDs consists of an exhaustive application of chase steps in a fair fashion, 
which leads to a collection of all instances $I'$ generated as described above, denoted as 
$\mathsf{chase}(D,\Sigma)$. Note that each instance of $\mathsf{chase}(D,\Sigma)$ is a model of $D\cup\Sigma$.

The above chase rule gives rise to the so-called chase sequence. A {\em chase sequence}:
$I_0\xrightarrow{\sigma_i,\,h_i}I_1$, $\ldots$, $I_{k}\xrightarrow{\sigma_{k},\,h_{k}}I_{k+1}$,
denotes the sequence of applications of the TGD chase rule such that:
(1) $I_0$ $=$ $D$; (2) for each $i$ $\in$ $\{1,\ldots,k\}$, 
$I_{i}\xrightarrow{\sigma_{i},\,h_{i}}I_{i+1}$ denotes the instance 
$I_{i+1}$ $=$ $I_{i}$ $\cup$ $\{h'_{i}(\mathsf{head}(\sigma_{i}))\}$ such that
assuming $\sigma_i$ $=$ $\varphi(\vect{X},\vect{Y})$ $\ra$ $\exists\vect{Z}\psi(\vect{X},\vect{Z})$,
then $h'_{i}$ is the extension of the homomorphism $h_{i}\REST_{\vect{X}}$ such that
$h_i(\varphi(\vect{X},\vect{Y}))$ $\subseteq$ $I_i$. 
Then lastly, for $k$ $\geq$ $1$, 
we denote by $\mathsf{chase}^{[k]}(D,\Sigma)$ as the resulting instance $I_k$
that is the result of the chase sequence: $I_0\xrightarrow{\sigma_i,\,h_i}I_1$, $\ldots$, $I_{k-1}\xrightarrow{\sigma_{k-1},\,h_{k-1}}I_{k}$. 

The notion \emph{level} in a chase is defined inductively as follows \cite{CaliGP12}: (1 ) for an atom $\alpha$ $\in$ $D$, we set $\textsc{level}(\alpha)$ $=$ $0$; then inductively,
(2) for an atom $\alpha$ $\in$ $\mathsf{chase}(D,\Sigma)$ obtained via some chase step $I_{k}\xrightarrow{\sigma,\,\eta}I_{k+1}$,
we set $\textsc{level}(\alpha)$ $=$ $\textsc{max}\big(\big\{\textsc{level}(\beta) \mid \beta\in\mathsf{body}(\sigma\eta)\}\big)$ $+$ $1$. Then finally,
for some given $k$ $\in$ $\mathbb{N}$, we set 
$\mathsf{chase}^k(D,\Sigma)$ $=$ $\big\{\alpha\mid\alpha\in\mathsf{chase}(D,\Sigma)\mbox{ and }\textsc{level}(\alpha)\leq k\big\}$.
Intuitively, $\mathsf{chase}^k(D,\Sigma)$ is the instance containing atoms that can be derived in a less than or equal to 
$k$ chase steps.

Given an atom $p(\vect{t})$ such that $\vect{t}$ $\in$ $(\Gamma$ $\cup$ $\Gamma_N)^{|\vect{t}|}$, 
we say that $\mathsf{chase}(D,\Sigma)$ {\em entails} $p(\vect{t})$ ($\mathsf{chase}^{[k]}(D,\Sigma)$ 
entails $p(\vect{t})$), denoted $\mathsf{chase}(D,\Sigma)$ $\models$ $p(\vect{t})$
($\mathsf{chase}^{[k]}(D,\Sigma)$ $\models$ $p(\vect{t})$, resp.),
iff there exists some atom of the same relational symbol $p(\vect{t'})$ $\in$ $\mathsf{chase}(D,\Sigma)$
($p(\vect{t'})$ $\in$ $\mathsf{chase}^{[k]}(D,\Sigma)$, resp.)
and a homomorphism $h:$ $\vect{t}$ $\longrightarrow$ $\vect{t'}$ such that 
$h(p(\vect{t}))$ $=$ $p(\vect{t'})$. 

\begin{theorem}\label{thm_chase_iff_q} \cite{CaliGP12}
    Given a BCQ $q$ over $\mathcal R$, a database $D$ for $\mathcal R$ and a  set $\Sigma$ of 
    TGDs over $\mathcal R$, $D\cup \Sigma\models q$ iff $\mathsf{chase}(D,\Sigma)\models q$.
\end{theorem}

\begin{definition}
[\bf{BDDP}]
\label{def-BDDP}
A class ${\mathcal C}$ of TGDs  
satisfies the {\em bounded derivation-depth property} (BDDP) if for each BCQ $q$ over a schema $\mathcal{R}$, for 
every input database $D$ for $\mathcal{R}$ and 
for every set $\Sigma\in {\mathcal C}$ over ${\mathcal R}$, $D\cup\Sigma\models q$ implies that there exists some $k\geq 0$ which
only depends on $q$ and $\Sigma$ such that $\mathsf{chase}^{k}(D,\Sigma)\models q$.
\end{definition}

It has been shown that the BDDP implies the first-order rewritability \cite{CaliGL12,CaliGP12}. 
Formally, the BCQA problem is {\em first-order rewritable} for a class $\mathcal{C}$ of sets of TGDs
if for each $\Sigma\in \mathcal{C}$, and each BCQ $q$, there exists a first-order query
$q_{\Sigma}$ such that $D\cup \Sigma\models q$ iff $D\models q_{\Sigma}$, for every input database $D$.
In this case, we also simply say that the class $\mathcal{C}$ of TGDs is {\em first-order rewritable}.

\section{Derivation Paths and Derivation Trees}

\comment{
For a given set $\Sigma$ of TGDs, by taking different input databases $D$, the 
chase procedure $\mathsf{chase}(D,\Sigma)$ generates different results. However,
it is not difficult to observe that atoms of $\mathsf{chase}(D,\Sigma)$ are actually 
generated by following certain derivation pattens embedded in the rules of 
$\Sigma$, which, in some sense, are independent from the input database $D$. In this section, we will provide a
characterization for this derivation property underlying every given $\Sigma$. 
}

First of all, to simplify our investigations, 
from now on, we will assume that for any given set $\Sigma$ of 
TGDs, 
each TGD $\sigma$ in $\Sigma$ is of a specific form: $\sigma$ has only one atom in the head where each 
existentially quantified variable occurs only once. That is, $\Sigma$ consists of the following rule:
\begin{eqnarray}
    \sigma: \varphi(\mathbf{X}, \mathbf{Y})\rightarrow \exists \mathbf{Z} p(\mathbf{X},\mathbf{Z}).
    \label{lb3}
\end{eqnarray}

\begin{theorem}\label{single_heads}
	Let $q$ be a BCQ over $\mathcal R$, $D$ a database for $\mathcal R$ and $\Sigma$  a set of 
	TGDs over $\mathcal R$. Then we have:
	\begin{enumerate}
	\item	
There exists a \textsc{LogSpace} construction of an atomic BCQ $q'$ 
	and a set of TGDs $\Sigma'$ of schema ${\cal R'}$ $\supseteq$ ${\cal R}$,
	where $|{\sf head}(\sigma')|$ $=$ $1$ for each $\sigma'$ $\in$ $\Sigma'$, 
	such that $D\cup\Sigma$ $\models$ $q$ iff  $D\cup\Sigma'$ $\models$ $q'$ \cite{CaliGP12}.
	\item 
	If $\Sigma'$ satisfies BDDP then $\Sigma$ also satisfies BDDP.
	\end{enumerate}
\end{theorem}
Under Theorem \ref{single_heads},
it is clear that 
considering such special TGDs of the form (\ref{lb3}) as well as the atomic BCQ $\exists \mathbf{Z}p(\mathbf{Z})$
will be sufficient, in the sense that 
all results related to these forms of TGDs and atomic BCQ can be carried over to the general case. 
%
So in the rest of this paper, we will only focus on these forms of TGDs and atomic BCQ in our study. 

\subsection{Comparability and derivation paths}

Let $\mathbf{t}=(t_{1}, \cdots, t_{1})$ and $\mathbf{t}'=(t'_{1},\cdots, t'_{n})$ 
    be two tuples of terms. We say that $\mathsf{t}$ and $\mathsf{t'}$ are 
    {\em type comparable} if $\mathsf{t}$ and $\mathsf{t'}$ satisfy the 
    following conditions: for each $i$ ($1\leq i \leq k$), 
        (1) constant $c$ $\in$ $\Gamma$, $t_{i}$ $=$ $c$ iff $t'_{i}$ $=$ $c$; (2)
        $t_{i}$ $\in$ $\Gamma_V$ iff $t'_{i}$ $\in$ $\Gamma_V$; and
        (3) $t_{i}$ $\in$ $\Gamma_N$ iff $t'_{i}$ $\in$ $\Gamma_N$.                           
    %
%
%
%
Intuitively, two tuples $\mathbf{t}$ and $\mathbf{t'}$ are type comparable if each
position between the two tuples agrees on the type of term they contain, i.e.,
constants are mapped to (the same) constants, variables to variables and labeled 
nulls into labeled nulls.

\begin{definition}[{\bf Position comparable tuples}]\label{comparable_def}
    Let $\mathbf{t}=(t_{1}, \cdots, t_{n})$ and $\mathbf{t'}=(t'_{1},\cdots, t'_{n})$ 
    be two tuples of terms of length $n$. 
    We say that $\mathbf{t}$ and $\mathbf{t'}$ 
    are {\em position comparable} (or simply called 
    {\em comparable}), denoted as $\mathbf{t}\sim\mathbf{t'}$, 
    if $\mathbf{t}$ and $\mathbf{t'}$ satisfy the following conditions: 
    \begin{enumerate}                               
        \item $\mathbf{t}$ and $\mathbf{t'}$ are \textit{type comparable};
        \item for each pair $(i,j)$ ($1\leq i<j\leq n$), $t_{i}$ $=$ $t_{j}$
              iff $t'_{i}$ $=$ $t'_{j}$;
        \item $t$ $\in$ $\big(\mathbf{t}$ 
              $\cap$ $\mathbf{t'}\big)$, $t_{i}=t$ iff $t'_{i}=t$ ($1\leq i \leq n$).
        %
        %
    \end{enumerate}
    We also use $\mathbf{t_1}\not\sim\mathbf{t}_2$ if it is not the case that 
    $\mathbf{t_1}\sim\mathbf{t}_2$.
\end{definition}

Under Definition \ref{comparable_def}, we have $(X,X',\mathsf{n})\sim(Z,Y,\mathsf{n'})$, but 
$(\mathsf{n},\mathsf{n}, \mathsf{n}',Z)\not\sim(\mathsf{n},\mathsf{n}',\mathsf{n}',W)$,
because in the latter, the null 
patterns in the first three positions of the two tuples are not ``comparable". 
\comment{
As another example, we also have that
$(\mathsf{n}_1,\mathsf{n}_2, \mathsf{n}',Z)\not\sim(\mathsf{n}_3,\mathsf{n}_4,W,\mathsf{n}')$
because the third and fourth positions of the tuples are not mapped to corresponding 
labeled nulls but instead to variables.
}

Let $X$ be a variable from $\Gamma_{V}$, and $t$ a term from 
$\Gamma$ $\cup$ $\Gamma_N$ $\cup$ $\Gamma_{V}$. 
A {\em binding} is an expression of the form $X/t$. In this case, we also say that $t$ is a 
binding of variable $X$. A {\em substitution} $[\mathbf{X}/\mathbf{t}]$
is a finite set of bindings containing at most one binding for each variable from $\mathbf{X}$. 
For a given tuple of terms $\mathbf{t}$, we apply a substitution $\theta$ to $\mathbf{t}$ and 
obtain a different tuple of terms, denoted as $\mathbf{t}\theta$.
For example, $(X,Y,\mathsf{n},W)[X/\mathsf{n'},Y/Y,W/Z]$ $=$ $(\mathsf{n'},Y,\mathsf{n},Z)$. 
For a quantifier-free formula $\varphi(\mathbf{X})$ and a substitution 
$\theta=[\mathbf{X}/\mathbf{t}]$, applying 
$\theta$ to $\varphi(\mathbf{X})$, i.e., $\varphi(\mathbf{X})\theta$, will result in formula  
$\varphi(\mathbf{t})$ which 
is obtained from $\varphi(\mathbf{X})$ by replacing each free variable $X$ by its corresponding 
binding from $\varphi(\mathbf{X})$. 

Now we define how a substitution is applied to an existential rule $\sigma$. 
We extend a substitution to existentially quantified variables.
We say that substitution $\theta$ $=$ $[\mathbf{X}/\mathbf{t}]$ 
is {\em applicable} to $\sigma$ if the arities 
of $\mathbf{X}$
in $\theta$ match the arities of 
the tuples of all universally and existentially quantified variables in $\sigma$, respectively.
We may write 
a substitution applicable to $\sigma$ as the form:
$\theta$ $=$ $[\mathbf{X}/\mathbf{t}_1,\mathbf{Y}/\mathbf{t}_2, \mathbf{Z}/\mathbf{n}]$. 
Then by applying $\theta$ to rule $\sigma$ of the form (\ref{lb3}), we will obtain a rule of the following form:
\begin{eqnarray}
    \sigma\theta: \varphi(\mathbf{t}_1,\mathbf{t}_2)\rightarrow p(\mathbf{t}_1,\mathbf{n}).
    \label{lb4}
\end{eqnarray}

\comment{

The motivation for extending a substitution to existentially quantified variables is quite clear. 
For a given set $\Sigma$ of TGDs,
we want to represent the underneath derivation of $\Sigma$ in a generic form so that 
such derivation may be instantiated by the 
chase procedure when a specific input database is taken into account. For this purpose, 
through a substitution, we not only
substitute those universally quantified variables in $\sigma$, but also intentionally
eliminate existentially quantified variables in $\mathsf{head}(\sigma)$ by replacing them with proper nulls.
In this way, atom $p(\mathbf{t}_1,\mathbf{n})$ may be used in triggering other rules of $\Sigma$ 
through further substitutions.
}

\begin{definition}[{\bf Derivation path}]\label{def2}
    Let $\Sigma$ be a set of TGDs. A {\em derivation path} $P$ of $\Sigma$ is a finite sequence 
    of pairs of an atom and a rule:
    \begin{eqnarray}
    (\alpha_1, \rho_1), \cdots, (\alpha_n,\rho_n),
    \label{lb2}
    \end{eqnarray}
    such that
    \begin{itemize}
        \item for each $1 \leq i \leq n$, $\alpha_i=\mathsf{head}(\rho_i)$;
        \item for each $1 \leq i \leq n$, $\rho_i=\sigma_i\theta_i$ for some 
        $\sigma_i\in \Sigma$ and substitution $\theta_i$; 
        \item for each $1 \leq i < n$, $\alpha_{i+1}\in \mathsf{body}(\rho_i)$; 
        \item for each $1 \leq i \leq n$, if a null $\mathsf{n}\in \mathsf{head}(\alpha_i)$ is 
        introduced due to the elimination of existentially quantified variable,
        then this $\mathsf{n}$  must not occur in $\rho_j$, for all 
        $j$ $\in$ $\{i+1,\ldots,n\}$.
    \end{itemize}
\end{definition}

\begin{example}\label{ex2}
    Consider a set $\Sigma$ of TGDs consisting of two rules: 
    \begin{quote}
        $\sigma_1: r(X,Y,Z)\rightarrow s(Y,X)$,\\
        $\sigma_2: s(X,Y)\rightarrow \exists Z\exists W r(Y,Z,W)$.
    \end{quote} 
    The following are three different derivation paths of $\Sigma$:
    \begin{quote}
        $P_1$: 
        \hspace*{.2in} $(s(\mathsf{n_1},Y_1), \sigma_1[X/Y_1,Y/\mathsf{n_1},Z/\mathsf{n_2}])$,\\
        \hspace*{.43in} $(r(Y_1,\mathsf{n_1},\mathsf{n_2}),\sigma_2[X/X_1,Y/Y_1,Z/\mathsf{n}_1, W/\mathsf{n}_2])$,\\
        \hspace*{.43in} $(s(X_1,Y_1),\sigma_1[X/Y_1,Y/X_1,Z/Z_1])$,\\
        $P_2$: 
        \hspace*{.2in} $(r(X_2,\mathsf{n_1},\mathsf{n_2}), \sigma_2[X/\mathsf{n_3},Y/X_2,Z/\mathsf{n}_1, W/\mathsf{n}_2])$, \\
        \hspace*{.43in} $(s(\mathsf{n_3},X_2),\sigma_1[X/X_2, Y/\mathsf{n_3}, Z/\mathsf{n_4}])$, \\
        $P_3$:
        \hspace*{.2in} $(r(X_2,\mathsf{n_1},\mathsf{n_2}), \sigma_2[X/\mathsf{n_3},Y/X_2,Z/\mathsf{n}_1,W/\mathsf{n}_2])$, \\
        \hspace*{.39in} $(s(\mathsf{n_3},X_2),\sigma_1[X/X_2, Y/\mathsf{n_3}, Z/\mathsf{n_4}])$, \\
        \hspace*{.40in} $(r(X_2,\mathsf{n_3},\mathsf{n_4}), \sigma_2[X/X_1,Y/X_2,Z/\mathsf{n}_3,W/\mathsf{n}_4])$. 
    \end{quote}
    $\Box$    
\end{example}

\begin{definition}[{\bf Generalising comparability relation}]
\label{eq-def}
We generalise the comparability relation $\sim$ defined earlier as follows.
\begin{enumerate}
\item Let $\sigma$ $\in$ $\Sigma$, and $\theta$ $=$ $[\mathbf{X}/\mathbf{t},\mathbf{Z}/\mathbf{n}]$ and
$\theta'$ $=$ $[\mathbf{X}/\mathbf{t'},\mathbf{Z}/\mathbf{n'}]$ be two substitutions applicable to 
$\sigma$. We say that $\sigma\theta$ and $\sigma\theta'$ are {\em comparable}, 
denoted as $\sigma\theta$ $\sim$ $\sigma\theta'$, 
if $\mathbf{tn}\sim\mathbf{t'n'}$.
\item
Let $P$ be a derivation path of $\Sigma$ of the form (\ref{lb2}), we use 
$|P|$ to denote its length. Furthermore, suppose $(\alpha_i, \rho_i)$ and 
$(\alpha_j, \rho_j)$ are two elements of $P$, we say that $(\alpha_i, \rho_i)$ 
and $(\alpha_j, \rho_j)$ are {\em comparable}, denoted 
as $(\alpha_i, \rho_i)\sim(\alpha_j, \rho_j)$, if $\rho_i\sim\rho_j$ (note that this implies
$\sigma_i=\sigma_j$).
\item
Let $P$ $=$ $((\alpha_1,\rho_1),(\alpha_2, \rho_2),\cdots)$ and
$P'= ((\alpha_1',\rho_1'), (\alpha_2', \rho_2'), \cdots)$ be two derivation paths of $\Sigma$. 
$P$ and $P'$ are {\em comparable}, denoted as  
$P\sim P'$, if $|P|=|P'|$ and for each $i$  ($1\leq i \leq |P|$), 
$(\alpha_i,\rho_i)\sim(\alpha_i',\rho_i')$.
\end{enumerate}
\end{definition}

It is easy to observe that $\sim$ defined in 
Definition \ref{eq-def} is an equivalence relation.
Although a derivation path may be infinitely long, the following result ensures that for any 
derivation path, it is sufficient to only consider its finite fragment.

%
%
%
\begin{proposition}[{\bf Derivation path length bound}]\label{th2}
    Let $\Sigma$ be a set of TGDs. Then 
    there exists a natural number $N$ such that for every derivation path $P$
    of the form (\ref{lb2}), if $|P|$ $>$ $N$ then there exists $i,j$ ($1\leq i<j\leq |P|$)
    such that $(\alpha_i, \rho_i)\sim(\alpha_j, \rho_j)$.            
\end{proposition}

\subsection{Derivation trees}


\begin{definition}[{\bf Derivation tree}]\label{def-tree}
    Given a set $\Sigma$ of TGDs. A {\em derivation tree} of $\Sigma$, denoted as $T(\Sigma)$, 
    is a finite tree $(N,E,\lambda)$, with nodes $N$, edges $E$ and labeling function $\lambda$, such that:
    \begin{enumerate}
        \item The nodes of $T(\Sigma)$ have \emph{labels} of the form $(\alpha, \rho)$,              
              where $\rho=\sigma\theta$ for some $\sigma\in\Sigma$ and $\theta$ a substitution, 
              and $\mathsf{head}(\rho)=\alpha$;
        \item For any node $v$ labeled by $(\alpha,\rho)$ of $T(\Sigma)$, let $\alpha_1,\cdots,\alpha_n$ be atoms in 
		      $\mathsf{body}(\rho)$, then $(\alpha,\rho)$ has $n$ children $v_1,\ldots,v_n$ labeled with 
		      $(\alpha_1,\rho_1)$, $\cdots$, $(\alpha_n,\rho_n)$, respectively, such that for each 
		      $i$ $\in$ $\{1,\ldots,n\}$, $\rho_i=\sigma_i\theta_i$ for some 
		      $\sigma_i\in \Sigma$ and $\theta_i$ a substitution, and $\mathsf{head}(\rho_i)=\alpha_i$;
        \item For any node $v$ labeled with $(\alpha, \rho)$ in $T(\Sigma)$, 
		      all "fresh" nulls occurring in $\alpha$, that are introduced through the 
              substitutions in $\rho$,
              must not occur in any labels of a descendant node of $v$;
        \item If node $v$ labeled with $(\alpha, \rho)$ is a leaf of $T(\Sigma)$, 
              then there does not exist any null $\mathsf{n}$ appearing in $\mathsf{body}(\rho)$.     
    \end{enumerate}         
    A {\em path} $P$ in $T(\Sigma)$, denoted as $P\in T(\Sigma)$, is a derivation path in 
    $T(\Sigma)$ starting from the root 
    and ending at a leaf. We define $\mathsf{depth}(T(\Sigma))=\mathsf{max}(\{|P|\mid P\in T(\Sigma)\}$ to 
    be the {\em depth} of $T(\Sigma)$. By $\mathsf{root}(T(\Sigma))$, $\mathsf{leafNodes}(T(\Sigma))$ 
    and $\mathsf{nodes}(T(\Sigma))$, we denote the root node, leaf nodes and all nodes of $T(\Sigma)$,
    respectively. Also, given some node $v$ of $T(\Sigma)$, we denote by 
    $\mathsf{childNodes}(v,T(\Sigma))$ (or just $\mathsf{childNodes}(v)$ when clear from the context) 
    as the child nodes of $v$ under the tree $T(\Sigma)$.
    Lastly, we use ${\mathcal T}(\Sigma)$ to denote the set of all derivation trees of $\Sigma$. 
\end{definition}

According to Definition \ref{def-tree}, each path of a derivation tree is a 
derivation path. Also, since $\Sigma$ can have an infinite number of possible
derivation paths due to possibly arbitrary number of repetitions of path fragments within a path
(i.e., a ``loop"), $\mathcal{T}(\Sigma)$ may contain an infinite number of derivation trees.

\begin{definition}[{\bf Derivation tree instantiation}]\label{def-dTree}
    Let $\Sigma$ be a set of TGDs, $D$ a database over schema ${\mathcal R}$, and
    $T(\Sigma)$ $=$ $(N,E,\lambda)$ a derivation tree of $\Sigma$.    
    Then we obtain a tree $T'$  $=$ $(N',E',\lambda')$ from $T(\Sigma)$, where $N\subseteq N'$, as follows: 
    %
    \begin{enumerate}
        \item For each leaf node $v$ in $T(\Sigma)$, where $\lambda(v)$ $=$ $(\alpha,\rho)$, do: 
        \begin{enumerate}
            \item Set $\lambda'(v)$ $=$ $(\alpha',\rho')$ in the tree $T'$,
                  where $\alpha'=\mathsf{head}(\rho')$, where $\rho' = \rho\theta$ 
                  for some substitution $\theta$ and $\mathsf{body}(\rho')$ $\subseteq$ $D$;                                        
            \item For each atom $\beta$ $\in$ $\mathsf{body}(\rho')$ $\subseteq$ $D$ with $\rho'$
                  as mentioned above, add a node $v'$ in $N'$ and set $\lambda'(v')$ $=$ $(\beta,\beta)$ 
                  and corresponding edge
                  $\LA v',v\RA$ in $E$ so that $v'$
                  is a leaf node (so now making $v$ a non-leaf node);                                                
        \end{enumerate}
        \item For a node $v$ such that $\lambda(v)$ $=$ $(\alpha,\rho)$, and where all the 
              label of its children have been 
              replaced as in 1 above (i.e., through ``$\lambda'$"), 
              set $\lambda'(v)$ $=$ $(\alpha',\rho')$, where $\rho'=\rho\theta'$ 
              for some substitution $\theta'$ such that
              for each atom $p(\mathbf{t})\in \mathsf{body}(\rho')$, either $p(\mathbf{t})\in D$ 
              or there exists a child node $v'$ of $v$ such that  
              $\lambda'(v')$ $=$ $(\alpha^{*},\rho^{*})$, where $p(\mathbf{t})=\alpha^{*}$;
        \item Continue 2, until no node can be further relabled.
    \end{enumerate}
    $T'$ is called an {\em instantiation} of $T(\Sigma)$ on $D$, denoted as $T(D,\Sigma)$, 
    if it does not contain any variables occurring in $T(\Sigma)$. Similarly to the case
    of derivation tree, we use $\mathsf{depth}(T(D,\Sigma))$ and $\mathsf{root}(T(D,\Sigma))$ to denote
    the {\em depth} and {\em root node}  of $T(D,\Sigma)$, respectively. 
    Finally, by ${\mathcal T}(D,\Sigma)$, we denote
    the set of all instantiations on $D$ for all derivation trees in ${\mathcal T}(\Sigma)$.
\end{definition}
For convenience from here on and when it is clear from the context, we will mostly refer to a node by its actual label,
e.g., a node $v$ $\in$ $N$ where $\lambda(v)$ $=$ $(\alpha, \rho)$ is simply refered to as
$(\alpha, \rho)$.

We say that an atom $p(\mathbf{t})$ {\em is supported} by
$T(D,\Sigma)$, denoted as $T(D,\Sigma)\models p(\mathbf{t})$, if 
$\lambda(\mathsf{root}(T(D,\Sigma)))$ $=$ $(\alpha,\rho)$ where $\alpha$ $=$ $p(\vect{s})$, and there is a
homomorphism $h$ such that $h(p(\vect{t}))$ $=$ $p(\vect{s})$.
The following result reveals an important relationship between the chase and derivation trees.

\begin{theorem}\label{c1}
    Let $\Sigma$ be a set of TGDs, $D$ a database over schema ${\mathcal R}$, and $q$ a BCQ query
    $\exists \mathbf{Z} p(\mathbf{Z})$.
    Then $\mathsf{chase}(D,\Sigma)\models q$ iff there exist an instantiation $T(D,\Sigma)$ for some derivation tree $T(\Sigma)$ and a substitution $\theta$, such that
    $T(D,\Sigma)\models p(\mathbf{t})$, where $\mathbf{t}$ is a tuple of terms from $\Gamma$
    of the same length as $\mathbf{Z}$, and $\mathbf{t}\theta=\mathbf{Z}$.
\end{theorem}
\begin{proof}
	(``$\Longrightarrow$") We prove this direction by first providing the following lemma.
	\begin{lemma}\label{tree_induction}
		Given an instantiated derivation tree $T(D,\Sigma)$ $=$ $(N,E,\lambda)$ with nodes $N$,
		edges $E$ and labeling function $\lambda$,
		of $\Sigma$ under a database $D$,
		there exists a homomorphism $\mu:$ $\mathsf{nodes}(T(D,\Sigma))$ $\longrightarrow$ $\mathsf{chase}^{[N]}(D,\Sigma)$, 
		where $N$ $\leq$ $|\mathsf{nodes}(T(D,\Sigma))|$ $-$ $|\mathsf{leafNodes}(T(D,\Sigma))|$,
		such that the following conditions are satisfied:
		
		\vspace*{-0.4cm}

		\begin{align}
			&\mbox{1. }\,\mbox{For each }v\in\mathsf{leafNodes}(T(D,\Sigma))
			\mbox{ such that }\lambda(v)=(\alpha,\alpha),\,\nonumber\\
			&\mu(v)=\alpha\in D;\label{homo_cond_1}
		\end{align}

		\vspace*{-0.5cm}

		\begin{align}
			&\mbox{2. }\,\mbox{For each }v\in\mathsf{nodes}(T(D,\Sigma))
			\mbox{ such that }\lambda(v)=(\alpha,\rho),\,\nonumber\\
			&\mathsf{child}(v)=\{v_1,\ldots,v_1\},\,\rho=\sigma\theta\mbox{ for some substitution }\theta,
			\nonumber\\
			&\mbox{and }\sigma=\varphi(\vect{X},\vect{Y})
			\ra\exists\vect{Z}p(\vect{X},\vect{Z})\in\Sigma,\,
			\mbox{ there exists a}\nonumber\\
			&\mbox{homorphism }h\mbox{ such that }h(\varphi(\vect{X},\vect{Y}))\subseteq
			\{\mu(v_1),\ldots,\mu(v_n)\}\nonumber\\
			&\mbox{and extension }h'\mbox{ of }h\REST_{\vect{X}}
			\mbox{ where }\mu(v)=h'(p(\vect{X},\vect{Z}));\label{homo_cond_2}
		\end{align}

		\vspace*{-0.5cm}
		
		\begin{align}
			&\mbox{3. }\,\mbox{For each }v\in\mathsf{nodes}(T(D,\Sigma))\mbox{ such that }
			\lambda(v)=(\alpha,\rho)\mbox{ and}\nonumber\\
			&\alpha=p(\vect{t}),\mbox{ if }\mu(v)=q(\vect{t'})
			\mbox{ then we have that }p(\vect{t})\theta=q(\vect{t'})\nonumber\\
			&\mbox{for some substitution }\theta.\label{homo_cond_3}                                 
		\end{align}

		\vspace*{-0.5cm}
		
		\begin{align}
			&\mbox{4. }\,\mbox{For each }v_1,v_2\in\mathsf{nodes}(T(D,\Sigma))\mbox{ such that }
			\lambda(v_1)=\lambda(v_2)\nonumber\\
			&\mbox{ (i.e., $v_1$ and $v_2$ have the same label)},\mbox{ then we also have that}\nonumber\\
			&\mu(v_1)=\mu(v_2).\label{homo_cond_4}
		\end{align}        
		\vspace*{-0.5cm}		
		
	\end{lemma}   
	\begin{proofsketch}
		We show the existence of such a homomorphism $\mu$ by induction on the 
		depth of the tree $T(D,\Sigma)$ starting from the leaf nodes 
		(i.e., the nodes labeld by the database facts) going up to the root node 
		labeled by $(\alpha,\rho)$.
	\end{proofsketch}
	
	\vspace*{-0.3cm}
	   
	Then from Lemma \ref{tree_induction}, since $T(D,\Sigma)$ $\models$ $p(\vect{t})$, then assuming that
	$\lambda\big(\mathsf{root}(T(D,\Sigma))\big)$ $=$ $(\alpha,\rho)$ such that 
	$\alpha$ $=$ $r(\vect{s})$, 
	we have from the definition of ``instantiated tree supportedness" of an atom 
	that $h(r(\vect{t}))$ $=$ $r(\vect{s})$ for 
	some homomorphism $h:$ $\vect{t}$ $\longrightarrow$ $\vect{s}$. Then because we have
	that $\mu(r(\vect{s}))$ $=$ $r(\vect{t'})$
	for some atom $r(\vect{t'})$ $\in$ $\mathsf{chase}^{[N]}(D,\Sigma)$, 
	where $N$ $\leq$ $|\mathsf{nodes}(T(D,\Sigma))|$ $-$ $|\mathsf{leafNodes}(T(D,\Sigma))|$
	and  
	$\mu:$ $\mathsf{nodes}(T(D,\Sigma))$ $\longrightarrow$ $\mathsf{chase}(D,\Sigma)$
	the ``bounding number" and  
	homomorphism defined in Lemma \ref{tree_induction}, respectively, then we also have
	from Lemma \ref{tree_induction} that
	$r(\vect{s})\theta$ $=$ $r(\vect{t'})$ for some substitution $\theta$.
	Therefore, with $h'$ $=$ $\theta\circ h$, then we have that 
	$h'(p(\vect{t}))$ $=$ $q(\vect{t'})$ $\in$ $\mathsf{chase}^{[N]}(D,\Sigma)$,
	which implies that $\mathsf{chase}^{[N]}(D,\Sigma)$ $\models$ $p(\vect{t})$.

	(``$\Longleftarrow$") Assume $\mathsf{chase}^{[N]}(D,\Sigma)$ $\models$ $p(\vect{t})$
	for some atom $p(\vect{t})$ and $N$ $\geq$ $1$. Then by the definition of 
	$\mathsf{chase}^{[N]}(D,\Sigma)$ $\models$ $p(\vect{t})$, there exists some atom
	$p(\vect{t'})$ $\in$ $\mathsf{chase}^{[N]}(D,\Sigma)$
	and homomorphism $h:$ $\vect{t'}$ $\longrightarrow$ $\vect{t}$  
	such that $h(p(\vect{t}))$ $=$ $p(\vect{t'})$. 
	Thus, there exists some finite chase sequence 
	$I_{0}$ $\xrightarrow{\sigma_0,h_0}$ $I_{1}$, $\ldots$, 
	$I_{N-1}$ $\xrightarrow{\sigma_N,h_N}$ $I_{N}$
	such that $p(\vect{t'})$ $\in$ $I_{N}$. Let us assume without loss of generality
	that for $i$ $\in$ $\{1,\ldots,N-1\}$, there does not exists another atom
	$p(\vect{t''})$ $\in$ $I_i$ such that $h(p(\vect{t}))$ $=$ $p(\vect{t''})$. Then based
	on the sequences of TGDs $\sigma_i$ and homomorphisms $h_i$ that made $\sigma_i$
	applicable to $I_i$, we can construct an instantiated derivation tree 
	$T(D,\Sigma)$ as follows: 
	\begin{align}
		&\mbox{1. }\mbox{Let }\mathsf{root}(T(D,\Sigma))\mbox{ be labeled with }(p(\vect{t'}),\sigma_{N}\theta_{N}),
		\mbox{ where}\nonumber\\
		&\theta_{N}\mbox{ is the }\mbox{corresponding substitution for }h_{N}\mbox{ and its extension }h'_{N};\nonumber
	\end{align}

	\vspace*{-0.3cm}	
	
	\begin{align}
		&\mbox{2. }\mbox{For each atom }\alpha\in
		\mathsf{chase}^{[N]}(D,\Sigma)
		\mbox{ either}:\nonumber\\                           
		&\bullet\,\,\mbox{add a node }v\mbox{ with label }(\alpha,\alpha),\mbox{ if }\alpha\in D,
		\mbox{ otherwise}\nonumber\\                 
		&\bullet\,\,\mbox{add a node }v\mbox{ with label }(\alpha,\rho),\mbox{ where }
		\alpha=\mathsf{head}(\rho),\nonumber\\
		&\hspace{0.5cm}\rho=\sigma_i\theta_i\mbox{ and }\theta_i\mbox{ the corresponding substitution for }h_i
		\nonumber\\
		&\hspace{0.5cm}\mbox{and its ``extension" }h'_i.\nonumber
	\end{align}

	\vspace*{-0.3cm}	
	
	\begin{align}                 
		&\mbox{3. }\mbox{For each node }v\mbox{ with label }(\alpha,\rho)\mbox{ such that }
		\rho=\sigma\theta,\nonumber\\
		&\mbox{for some }\sigma\in\Sigma\mbox{ and substitution }
		\theta,\mbox{ and }\mathsf{body}(\rho)=\nonumber\\
		&\{\alpha_1,\ldots,\alpha_n\}\mbox{ then }\mbox{for }i\in\{1,\ldots,n\},\mbox{ add an edge }
		(v,v_i)\nonumber\\
		&\mbox{ such that either:}\nonumber
	\end{align}

	\vspace*{-0.4cm}	
	
	\begin{align}
		&\bullet\,\,v_i\mbox{ is labled with }(\alpha_i,\alpha_i),\mbox{ if }\alpha_i\in D,
		\mbox{ otherwise}\nonumber\\                 
		&\bullet\,\,v_i\mbox{ is labled with }(\alpha_i,\rho_i),\mbox{ such that }
		\alpha_i=\mathsf{head}(\rho_i),\nonumber\\
		&\hspace{0.5cm}\rho_{i}=\sigma_{j}\theta_{j},\,\theta_{j}\mbox{ the corresponding subtitution }
		\mbox{ for }h_{j}\nonumber\\
		&\hspace{0.5cm}\mbox{ (and corresponding extension }h'_{j}\mbox{) and }
		I_{j}\xrightarrow{\sigma_j,h_j}I_{j+1}\nonumber\\
		&\hspace{0.5cm}\mbox{is the first chase step that derived }\alpha_i.\nonumber                                   
	\end{align}

	\vspace*{-0.1cm}	
	
	Then it is not too difficult to see that the above construction for
	$T(D,\Sigma)$ is in fact an instantiated derivation tree and where
	$N$ $\leq$ $|\mathsf{nodes}(T(D,\Sigma))|$ $-$ $|\mathsf{leafNodes}(T(D,\Sigma))|$.
	(i.e., recall that $p(\vect{t'})$ $\in$ $I_{N}$ such that 
	$I_{N-1}$ $\xrightarrow{\sigma_N,h_N}$ $I_{N}$ is the first chase step
	that derived $p(\vect{t'})$).
	Therefore,
	because $h(p(\vect{t}))$ $=$ $p(\vect{t'})$ for some homomorphism
	$h:$ $\vect{t}$ $\longrightarrow$ $\vect{t'}$
	\big(i.e., recall that $\mathsf{chase}(D,\Sigma)$ $\models$ $p(\vect{t})$
	and $p(\vect{t'})$ $\in$ $\mathsf{chase}$ such that $h(p(\vect{t}))$ $=$ $p(\vect{t'})$\big)
	and since, assuming that
	$(\alpha,\rho)$ $=$ $\mathsf{root}(T(D,\Sigma))$, we have that 
	$p(\vect{t'})$ $=$ $\alpha$ from the construction of $T(D,\Sigma)$, 
	then we clearly have that $T(D,\Sigma)$ $\models$ $p(\vect{t})$
	through the same ``witnessing" homomorphism $h$.
\end{proof}

\section{Loop Restricted (LR) TGDs }

Theorem \ref{c1} shows that derivation trees provide a precise characterisation 
for the chase procedure. 
Therefore, the query answering against a set of TGDs together with an input database 
can be achieved by 
computing and checking the corresponding instantiation of the underlying derivation tree.
However, since the derivation tree for a given set of TGDs may
be of an arbitrary depth, this process is generally undecidable.

In this section, we will define a new class of TGDs, named {\em loop restricted} (LR)
TGDs, such that the depth of all derivation trees for this type of 
TGDs is always bounded in some sense. From this result, we will further prove that 
LR TGDs satisfy the 
{\em bounded derivation-depth property} (BDDP) \cite{CaliGP12}.


\begin{definition}[{\bf Loop pattern}]\label{loop-p}
    Let $P$ $=$ $((\alpha_1,\rho_1)$, $\cdots$, $(\alpha_n,\rho_n))$ be a derivation path as 
    defined in Definition \ref{def2}. Then $P$ is a {\em loop pattern} if 
    $(\alpha_1,\rho_1)$ $\sim$ $(\alpha_n,\rho_n)$ and $(\alpha_i,\rho_i)\not\sim(\alpha_j,\rho_j)$ 
    for any other $i,j$ ($1 < i, j < n$).
\end{definition}

Let $L$ be a loop pattern as defined in Definition \ref{loop-p}. 
For each pair $(\alpha_i,\rho_i)$ in $L$ where $1\leq i<n$, an atom $\beta\in \mathsf{body}(\rho_i)$ is called 
{\em recursive atom} if $\beta=\alpha_{i+1}$ for $(\alpha_{i+1},\rho_{i+1})$.

\begin{example}\label{ex6}
    Example \ref{ex2} continued.
    It is easy to see that derivation paths $P_1$ and $P_3$ are 
    loop patterns, while $P_2$ is not. 
    Furthermore, $P_1$ and $P_3$ are the only two different
    loop patterns of the given $\Sigma$, considering that for all other loop patterns $P$, 
    it will be either $P\sim P_1$ or $P\sim P_3$\footnote{See Definition \ref{eq-def} for
    derivation path (loop pattern) comparability relation.}. $\Box$
\end{example}

%
\begin{proposition}
\label{pro-loop}
    Given a finite set $\Sigma$ of TGDs, $\Sigma$ only has a finite number of 
    loop patterns under equivalence relation $\sim$.    
\end{proposition}
\comment
{
\begin{proof}
    From Proposition \ref{th2}, there is a number $N$ such that for any derivation path 
    $P$ of the form (\ref{lb2}), for each $N+1\leq j\leq |P|$, there exists some
    $1\leq i\leq N$ such that $(\alpha_i,\rho_i)$ $\sim$ $(\alpha_j,\rho_j)$.
    Therefore, it follows that one only has to check each derivation path $P$
    if it is a loop pattern.
\end{proof}
}


\subsection{Restricted loop patterns}

\begin{example}
\label{ex4.3.3}
Consider Example \ref{ex1} in Introduction again. Here we simplify $\Sigma_{\mathsf{Research}}$ by removing $\sigma_3$ and renaming predicates in $\sigma_1$ and $\sigma_2$. Note that such change will not affect 
 $\Sigma_{\mathsf{Research}}$'s original loop pattern.
    \begin{quote}
        $\sigma_1: q(X,Y)\rightarrow p(X)$,\\
        $\sigma_2: p(X),r(X,Y), s(X,Y,Y)\rightarrow \exists W q(X,W)$.
    \end{quote}
We can verify that $\Sigma$ does not belong to any of currently known first-order rewritable 
TGDs classes. 

    Now we consider the derivation of atom $q(X_1,\mathsf{n}_1)$ from $\Sigma$. 
    The following are two different derivation trees for 
    $q(X_1,\mathsf{n}_1)$, and both involve recursive calls to $\sigma_1$ and $\sigma_2$:
{\small
    \begin{align}
        T_1:&\nonumber\\
        w_0^{1}&=(\alpha_0^{1},\rho_0^{1})\nonumber\\
               &=(q(X_1,\mathsf{n}_1), [p(X_1),r(X_1,X_2),s(X_1,X_2,X_2)\rightarrow q(X_1,\mathsf{n}_1)]),\nonumber\\
        w_1^{1}&=(\alpha_1^{1},\rho_1^{1})\nonumber\\
               &=(p(X_1), [q(X_1,\mathsf{n}_2)\rightarrow p(X_1)]),\nonumber\\
        w_2^{1}&=(\alpha_2^{1},\rho_2^{1})\nonumber\\
               &= (q(X_1,\mathsf{n}_2), [p(X_1),r(X_1,X_3),s(X_1,X_3,X_3)\rightarrow q(X_1,\mathsf{n}_2)]);\nonumber
    \end{align}
    \vspace*{-.3in}
    \begin{align}
        T_2:&\nonumber\\
        w_0^{2}&=(\alpha_0^{2},\rho_0^{2})\nonumber\\
               &=(q(X_1,\mathsf{n}_1), [p(X_1),r(X_1,X_2),s(X_1,X_2,X_2)\rightarrow q(X_1,\mathsf{n}_1)]),\nonumber\\
        w_1^{2}&=(\alpha_1^{2},\rho_1^{2})\nonumber\\
               &=(p(X_1), [q(X_1,\mathsf{n}_2)\rightarrow p(X_1)]),\nonumber\\  
        w_2^{2}&=(\alpha_2^{2},\rho_2^{2})\nonumber\\
               &=(q(X_1,\mathsf{n}_2), [p(X_1),r(X_1,X_3), s(X_1,X_3,X_3)\rightarrow q(X_1,\mathsf{n}_2)]),\nonumber\\
        w_3^{2}&=(\alpha_3^{2},\rho_3^{2})\nonumber\\
               &=(p(X_1), [q(X_1,X_4)\rightarrow p(X_1)]).\nonumber             
    \end{align} 
}
    Intuitively,    
    $T_1$ simply contains one derivation path $L_1=w_0^{1}w_1^{1}w_2^{1}$ which is also a loop pattern, while
    $T_2$'s loop pattern is $L_2= w_1^{2}w_2^{2}w_3^{2}$.  
    %
    If we consider all other derivation trees for atom $q(X_1,\mathsf{n}_1)$, 
it is not difficult to observe that all these trees {\em are subsumed} by either $T_1$ or $T_2$, in the sense
that derivations
illustrated in $T_1$ or $T_2$ sufficiently cover those illustrated in all other trees
     
$\Sigma$ 
presents an interesting case of satisfying the so-called {\em bounded derivation tree depth property} (BDTDP) (the definition will be given later). 
By examining the two loop patterns, we find that 
they can be split in such a way where all variables in the recursive atoms are bounded by the variables occurring 
in the heads of all corresponding rules. This will make the derived atom in each derivation step from
the corresponding derivation tree
not rely on any new variables in recursive atoms.

Consider loop pattern $L_1$, for instance, 
for each pair $(\alpha_i^{1},\rho_i^{1})$ ($i=0,1,2$), we can split the set $\mathsf{body}(\rho_i^{1})$ of atoms in the
body of $\rho_i^{1}$ into two disjoint parts
$\mathsf{body_h}(\rho_i^{1})$ and $\mathsf{body_b}(\rho_i^{1})$, such that the common variables 
in $\alpha_{i}^{1}\cup \mathsf{body_h}(\rho_i^{1})$ and $\mathsf{body_b}(\rho_i^{1})$ are exactly the common variable 
occurring in all $\alpha_{i}$, which is $X_1$, whilst the underlying
recursive atoms in the loop pattern only occur 
in $\mathsf{body_b}(\rho_i^{1})$, i.e., $\alpha_{i+1}^{1}\in \mathsf{body_b}(\rho_i^{1})$ for $i=0,1$. 
We can do a similar separation for loop pattern $L_2$ as well.
As will be showed next, it turns out that a set $\Sigma$ of TGDs having this 
feature always ensures BDTDP. $\Box$
\end{example}

Now we are ready to formally define the notion of restricted loop patterns. 
Let $A$ be a set of atoms, we use $\mathsf{var}(A)$ to denote the set of all variables occurring in $A$.

\begin{definition}[{\bf Loop restricted (LR) patterns}]\label{def-restricted}
    Let $\Sigma$ be a set of TGDs. $\Sigma$ is {\em loop restricted} (LR),
    if for each loop pattern $L= (\alpha_1,\rho_1) \cdots (\alpha_n, \rho_n)$ 
    of $\Sigma$, $L$ satisfies the following conditions: 
    %
%
        %
        %
for each pair $(\alpha_i,\rho_i)$ in $L$ ($1\leq i <n$),  the set of atoms $\mathsf{body}(\rho_i)$ 
        can be separated into two disjoint parts $\mathsf{body}(\rho_i)=\mathsf{body_h}(\rho_i)\cup \mathsf{body_b}(\rho_i)$, such that
        (1) $\mathsf{body_h}(\rho_i)\cap \mathsf{body_b}(\rho_i)=\emptyset$, 
        (2) $\alpha_{i+1}\in \mathsf{body_b}(\rho_i)$, and
        (3) $\mathsf{var}(\{\alpha_i\}\cup \mathsf{body_h}(\rho_i))$
            $\cap$ $\mathsf{var}(\mathsf{body_b}(\rho_i))$ $=$ $\bigcap_{j=1}^{n}\mathsf{var}(\alpha_j)$.
        \comment{
        \begin{itemize}
            \item[-] $\mathsf{body_h}(\rho_i)\cap \mathsf{body_b}(\rho_i)=\emptyset$,
            \item[-] $\alpha_{i+1}\in \mathsf{body_b}(\rho_i)$, 
            \item[-] $\mathsf{var}(\{\alpha_i\}\cup \mathsf{body_h}(\rho_i))$
                     $\cap$ $\mathsf{var}(\mathsf{body_b}(\rho_i))$ $=$ $\bigcap_{j=1}^{n}\mathsf{var}(\alpha_j)$.
        \end{itemize}
        }
        %
    %
\end{definition}

\begin{example}
Example \ref{ex4.3.3} continued.
It is easy to see that loop patterns $L_1$ and $L_2$ in Example \ref{ex4.3.3} satisfy the conditions  
of Definition \ref{def-restricted}. Furthermore, if we consider the derivation of atom $p(X)$ from $\Sigma$, the underlying
loop patterns deduced from its derivations also satisfy the conditions of Definition \ref{def-restricted}. So 
$\Sigma$ is loop restricted. 
$\Box$
\end{example}

\subsection{Main results}\label{main_results}

Now we study the main properties of the new class LR TGDs. We first define
a property called bounded derivation tree depth property (BDTDP).

\begin{definition} [\bf{BDTDP}]
\label{def-BDTDP}
    A class ${\mathcal C}$ of TGDs satisfies the 
    \emph{bounded derivation tree depth property} (BDTDP) 
    if for each $\Sigma\in \mathcal{C}$, there exists some 
    $k\geq 0$ such that for every BCQ query $\exists\vect{Z}p(\vect{Z})$ and every database $D$, 
    $D\cup\Sigma$ $\models$ $\exists\vect{Z}p(\vect{Z})$
    iff $T(D,\Sigma)$ $\models$ $p(\vect{n})$ for some instantiated derivation
    tree $T(D,\Sigma)$ and atom $p(\vect{n})$, where
    $\mathsf{depth}(T(D,\Sigma))$ $\leq$ $k$ and $h(\vect{Z})$ $=$ $\vect{n}$ 
    for some homomorphism $h$. 
\end{definition}

Basically, Definition \ref{def-BDTDP} says
 that if a class of TGDs satisfies BDTDP, then its every BCQ query answering problem 
 can be always decided within a fixed 
 number $k$ of derivation steps with respect to the corresponding instantiated derivation
    trees. Note that this $k$ is independent from the input database $D$ and the specific BCQ query $q$.
Also note that BDTDP is different from the previous BDDP, i.e.,
Definition \ref{def-BDDP}, which is defined
based on the chase procedure.

\begin{theorem} 
\label{main}
    The class of LR TGDs satisfies BDTDP.
\end{theorem}
\begin{proof}
	We first introduce the notion of
	subsumation between two derivation trees.
	\begin{definition}[{\bf Derivation tree subsumption}]\label{def-subsumed}
		Let $\Sigma$ be a set of TGDs, and
		$T_1(\Sigma)$ and $T_2(\Sigma)$ be two derivation trees of $\Sigma$.
		Then we say that $T_2(\Sigma)$ \emph{subsumes} $T_1(\Sigma)$ 
		if the following conditions are satisfied: 
		(1) $\mathsf{root}(T_2(\Sigma))$ $=$ $\mathsf{root}(T_1(\Sigma))$; and
		(2) $\mathsf{leafNodes}(T_2(\Sigma))$ $\subset$ $\mathsf{leafNodes}(T_1(\Sigma))$.            
	\end{definition} 

	\vspace*{-0.5cm}

	\begin{proof} 
		Given a set $\Sigma$ of LR TGDs.   
		Let $\mathcal{T}(\Sigma)$ be the set of all derivation trees of $\Sigma$. We consider
		the set $\mathbb{T}(\Sigma)$ of all derivation trees 
		that are 
		distinct under $\sim$
		and their tree depths are not larger than $N$, where
		$N$ is the integer mentioned in Proposition 2\footnote{A 
			complete proof of Proposition 2 is given in the full version of this paper, in which $N$ is presented.}.
		Then it is clear that $\mathbb{T}(\Sigma)\subseteq \mathcal{T}(\Sigma)$ and is a finite
		set.       %
		Now we can prove the following important result:
		
		\vspace*{-.05in}
		
		\begin{lemma}
			\label{l1}
			Let $T(\Sigma)\in \mathcal{T}(\Sigma)$ (note $\Sigma$ is LR). Then for every database $D$ and every atom
			$p(\mathbf{t})$, 
			$T(D,\Sigma)\models p(\mathbf{t})$ iff there exists some
			$T'(\Sigma)\in \mathbb{T}(\Sigma)$ such that $T'(D,\Sigma)\models p(\mathbf{t})$.
		\end{lemma}
		
		\vspace*{-.05in}
		
		Then the theorem follows directly from Lemma \ref{l1}, by setting the 
		bound to be the maximal depth of trees in $\mathbb{T}(\Sigma)$. 
		The key idea of proving Lemma \ref{l1} is 
		based on the fact that for any tree $T(\Sigma)$
		in $\mathcal{T}(\Sigma)$, there is a corresponding 
		tree $T'(\Sigma)$ in $\mathbb{T}(\Sigma)$ which can replace 
		$T(\Sigma)$ without affecting $T(\Sigma)$'s derivations. 
		Without loss of generality, consider a tree $T(\Sigma)$ in $\mathcal{T}(\Sigma)$, where
		a path $P$ in $T(\Sigma)$ is longer than $N$.
		Then from Proposition 2, there must exist a loop pattern
		$L=(w_i, \cdots, w_j)$ in path $P$, such that the depth of node $w_i$ 
		is within the bound $N$, and the depth of node $w_j$ 
		is beyond $N$.
		Since $w_i\sim w_j$ and $L$ is loop restricted and from the conditions presented in Definition 8,
		then using similar ideas from \cite{ChenLZZ11}, we can prove that 
		the subtree underneath the node $body_{b}(\rho_i)$ in $T(\Sigma)$ can be replaced by
		the subtree underneath the node $body_{b}(\rho_j)$.  That is, the 
		loop pattern fragment $(w_i,\cdots, w_j)$ in path $P$ is replaced by
		a new node $w_i^{*}: (\alpha_i,[body_{b}(\rho_j), body_{h}(\rho_i)\ra \alpha_i])$. 
		According to Proposition 2, $\Sigma$ only
		has a finite number of loop patterns under $\sim$. So by
		doing this {\em folding} for all
		paths in $T(\Sigma)$, we eventually transform $T(\Sigma)$ into a $T'(\Sigma)$ whose depth is bounded by $N$, that is, 
		$T'(\Sigma)\in \mathbb{T}(\Sigma)$.
	\end{proof}	
	
	\vspace*{-0.37cm}
	
\end{proof}

\vspace*{-0.2cm}

The following theorem reveals an important connection between BDTDP and BDDP.

\begin{theorem}\label{th4}
If a class ${\cal C}$ of TGDs satisfies BDTDP then ${\cal C}$ also satisfies BDDP.
\end{theorem}

According to Theorem \ref{th4} from \cite{CaliGP12}, it is clear that 
the class of LR TGDs is first-order rewritable. 

\begin{theorem}
\label{comb_complexity}
For the class of LR TGDs, the BCQA's
 data complexity is in $\textsc{AC}^{0}$, and the
    combined complexity is \textsc{ExpTime} complete.
\end{theorem}

\begin{theorem}
\label{membership-complexity}
   Deciding whether a set of TGDs is loop restricted 
is \textsc{Pspace} complete.    
\end{theorem}

\section{Generalisation of Loop Restricted Patterns}

As described in previous section, the notion of loop patterns provides a useful
means of defining the class of LR TGDs that is first-order rewritable. 
Now we show that loop patterns can be employed as a unified notion to 
significantly extend LR TGDs to a more general class of TGDs. 

Firstly, we introduce a useful notion.
Let $(\alpha, \rho)$ be in a loop pattern $L$ and $B$ a set of atoms occurring in 
$\mathsf{body}(\rho)$. We use notion $\mathsf{null}(B)$ 
to denote the set of all labelled
nulls occurring in $B$. 


\begin{definition}[{\bf Generalised loop restricted (GLR) patterns}]
\label{def-generalised-restricted}
    Let $\Sigma$ be a set of TGDs. $\Sigma$ is {\em generalised loop restricted} (GLR),
    if each loop pattern $L= (\alpha_1,\rho_1)$ $\cdots$ $(\alpha_n, \rho_n)$ 
    of $\Sigma$ falls into one of the following four types:
    \begin{description}
    	\item[Type I] For each pair $(\alpha_i,\rho_i)$ in $L$ ($1\leq i <n$), $\mathsf{body}(\rho_i)$ 
		      can be separated into two disjoint parts $\mathsf{body}(\rho_i)=\mathsf{body_h}(\rho_i)\cup \mathsf{body_b}(\rho_i)$ such that the following three conditions holds: \label{GLR_1}
		      \begin{enumerate}
		          \item	$\mathsf{body_h}(\rho_i)\cap \mathsf{body_b}(\rho_i)=\emptyset$,
		          \item	$\alpha_{i+1}\in \mathsf{body_b}(\rho_i)$,
		          \item $\mathsf{var}\big(\{\alpha_i\}\cup \mathsf{body_h}(\rho_i)\big)$
		                $\cap$ $\mathsf{var}\big(\mathsf{body_b}(\rho_i)\big)$ $=$
		                $\bigcap_{j=1}^{n}\mathsf{var}(\alpha_j)$;	        
		     \end{enumerate}
    	\item[Type II] There exists a pair $(\alpha_i,\rho_i)$ in $L$ ($1\leq i <n$) such that $\mathsf{body}(\rho_i)$ 
		      can be separated into two disjoint parts $\mathsf{body}(\rho_i)=\mathsf{body_h}(\rho_i)\cup \mathsf{body_b}(\rho_i)$, where the following three conditions hold:\label{GLR_2}
		      \begin{enumerate}
		          \item	$\mathsf{body_h}(\rho_i)\cap \mathsf{body_b}(\rho_i)=\emptyset$,
		          \item	$\alpha_{i+1}\in \mathsf{body_b}(\rho_i)$,
		          \item $\mathsf{var}\big(\{\alpha_i\}\cup \mathsf{body_h}(\rho_i)\big)$
		                $\cap$ $\mathsf{var}\big(\mathsf{body_b}(\rho_i)\big)$ $=$
		                $\emptyset$;	          
		      \end{enumerate}
		\item[Type III] For each pair $(\alpha_i,\rho_i)$ in $L$ ($1\leq i< n$) and each
		    $\beta$ $\in$ ${\sf body}(\rho_i)$, 
		      ${\sf var}(\rho_i)$ $\subseteq$ ${\sf var}(\beta)$;\label{GLR_3}
		\item[Type IV] For each pair $(\alpha_i,\rho_i)$ in $L$ ($1\leq i< n$) and each
		$\beta$ $\in$ ${\sf body}(\rho_{i})\setminus\{\alpha_{i+1}\}$,  
		$\big({\sf var}(\alpha_{i+1})$ $\cap$ ${\sf var}(\beta)\big)\neq \emptyset$ implies
			  $\big({\sf var}(\alpha_{i+1})$ $\cap$ ${\sf var}(\beta)\big)$ 
			  $\subseteq$ $\bigcap_{j=1}^i{\sf var}(\alpha_j)$;
			 \label{GLR_4} 	
	\comment{
	\item[Type V]		 
			 For each pair $(\alpha_i,\rho_i)$ in $L$ ($1\leq i< n$), there exist some
		$\beta$ $\in$ ${\sf body}(\rho_{i})\setminus\{\alpha_{i+1}\}$ and a substitution $\theta$, such that
		$({\sf Evar}(\sigma_{i+1}))\theta \cap {\sf var}(\beta)\neq \emptyset$\footnote{Recall 
		that $\rho_{i+1}=\sigma_{i+1}\theta_{i+1}$.}.
		\label{GLR-5}
		}
		\item[Type V] There exists a pair $(\alpha_i,\rho_i)$ in $L$ ($1\leq i< n$), such that $\mathsf{body}(\rho_i)$ 
		      can be separated into two disjoint parts $\mathsf{body}(\rho_i)=\mathsf{body_h}(\rho_i)\cup \mathsf{body_b}(\rho_i)$, where the following three conditions hold: \label{GLR_5}
		      \begin{enumerate}
		      \item 	$\mathsf{body_h}(\rho_i)\cap \mathsf{body_b}(\rho_i)=\emptyset$,
		      \item $(\bigcup_{j=i+1}^n (\alpha_j))\cap \mathsf{body_h}(\rho_i)=\emptyset$,
		      \item $\mathsf{null}(\mathsf{body_h}(\rho_i))\neq \emptyset$.
		      \end{enumerate}
    \end{description}		        
\end{definition}

\comment{
The following Proposition \ref{glr-contains} now reveals how the GLP captures some of the other first-order
rewritable classes known in the literatures. For this reason, it will be convenient to first
introduce the following notions: we denote by aGRD, ML and SJ the classes of 
\textit{acyclic graph of rule dependencies} (aGRD) \cite{Baget04}, 
\textit{multilinear} (ML) \cite{CaliGK08} and \textit{sticky-join} \cite{CaliGP12} classes
of TGDs, respectively.
}

Let us take a closer look at Definition \ref{def-generalised-restricted}. Firstly, Type I simply specifies
LR TGDs, so the class of GLR TGDs properly contains the class of LR TGDs.
Type II 
says that for the body part of $\rho_i$ containing the recursive atom in the loop pattern, i.e.,  
$\alpha_{i+1}\in \mathsf{body_b}(\rho_i)$, its variables are not in common with
variables occurring in the head $\alpha_i$ and the other part of the body 
$\mathsf{body_h}(\rho_i)$. This indicates that
recursion embedded in the underlying loop pattern will not actually happen due to the lack of
shared variables. 

Type III, on the other hand, says that for each rule $\rho_i$ in every loop pattern,
all variables occurring in $\rho_i$ are guarded by each atom in $\rho_i$'s body.
Type IV concerns the shared variables occurring in both recursive and non-recursive 
atoms in the body of rule $\rho_i$ in a loop pattern,
i.e.,
${\sf var}(\alpha_{i+1})$ $\cap$ ${\sf var}(\beta)$. 
It requires that all such shared variables must be passed on to all following rules in the loop pattern.
Finally, Type V ensures that no cycle occurs in $\Sigma$'s graph of rule dependencies. 

\comment{
\begin{proposition}\label{glr-contains-lr}
    Let GLR denote the class of \textit{generalised loop restricted} TGDs as defined in Definition \ref{def-generalised-restricted}, then we have \emph{LR} $\subsetneq$ \emph{GLR}.      	
\end{proposition}
}

\begin{theorem}\label{GLR_BDTDP_property}
	The class of GLR TGDs satisfies BDTDP.
\end{theorem}
According to Theorem \ref{th4}, we know that the class of GLR TGDs satisfying BDTDP also satisfies
BDDP, and hence the following corollary holds.

\begin{corollary}
	The class of GLR TGDs is first-order rewritable.	
\end{corollary}

\begin{theorem}
\label{glr_comb_complexity}
Consider the BCQA problem for a given set of GLR TDGs.  
Its data complexity is in $\textsc{AC}^{0}$, and its
    combined complexity is \textsc{ExpTime} complete.
\end{theorem}

\begin{theorem}
\label{glr_membership-complexity}
   Deciding whether a set of TGDs is generalised loop restricted 
is \textsc{Pspace} complete.    
\end{theorem}

\section{Relationship to Other First-order Rewritable Classes}

In this section, we study the relationship between our proposed GLR TGDs class and other first-order rewritable
TGDs classes. First of all, we briefly introduce these existing TGDs classes, which are known to be 
first-order rewritable.
 A TGD of the form (\ref{lb1}):
\begin{quote}
$\sigma: \forall \mathbf{XY} \varphi(\mathbf{X}, \mathbf{Y})\rightarrow \exists\mathbf{Z}\psi(\mathbf{X},\mathbf{Z})$
\end{quote}
 is called {\em linear} if $\varphi(\mathbf{X}, \mathbf{Y})$ is an atom. $\sigma$ is {\em multi-linear}
 if each atom in $\varphi$ contains all the universally quantified variables of $\sigma$ \cite{CaliGL12}.
 A set $\Sigma$ of TGDs is linear or multi-linear if each TGD in $\Sigma$ is linear or multi-linear, respectively.
 $\Sigma$ is {\em acyclic} if $\Sigma$'s position graph contains no cycle \cite{wr-2012}, while
$\Sigma$ is {\em aGRD}
 if $\Sigma$'s rule dependency graph contains no cycle
\cite{Baget04,BagetLMS11}.

Informally, $\Sigma$ is said to have the {\em sticky property} if for each $\sigma$ in $\Sigma$, all variables occurring in ${\mathsf body}(\sigma)$ more than once also appear in $\mathsf{head}(\sigma)$, 
and furthermore, also appear in every atom obtained from some chase derivation which involves $\mathsf{head}(\sigma)$, that is, {\em stick} to all such atoms \cite{CaliGP12}. The {\em sticky-join property}, on the other hand,
 is less restricted than sticky property, where it only requires to stick certain variables 
 occurring more than once in $\mathsf{body}(\sigma)$ based on certain joinless condition. It has been showed that
 the sticky-join class captures both the sticky and linear classes, but is incomparable with
 multi-linear class \cite{CaliGP12}.

GLR actually captures a large class of first-order rewritable TGDs. 
Let us use LR, ML, AC, SJ, aGRD and DR to denote the classes of loop restricted, mulit-linear \cite{CaliGL12}, acyclic \cite{wr-2012},  
sticky-join \cite{CaliGP12}, aGRD \cite{Baget04,BagetLMS11} and domain restricted TGDs \cite{BagetLMS11}, respectively. 
Then we have the following result.
\begin{proposition}\label{glr-contains}
    Let GLR  be the class of \textit{generalised loop restricted} TGDs defined in Definition 
    \ref{def-generalised-restricted}. Then we have that: (1) \emph{LR} $\subsetneq$ \emph{GLR};
    (2) \emph{AC} $\subsetneq$ \emph{GLR};
    (3) \emph{ML} $\subsetneq$ \emph{GLR}; (4) \emph{SJ} $\subsetneq$ \emph{GLR}; 
    (5) \emph{aGRD} $\subsetneq$ \emph{GLR}; (6) \emph{DR} $\subsetneq$ \emph{GLR}.
   %
    %
\end{proposition}
\begin{proof}
	We prove by considering the individual cases as follows:
	\begin{description}
		\item[](``DR $\subsetneq$ GLR"):
			This follows from the fact that a TGD rule $\sigma$ is \emph{domain-restricted}
			if each head atom $\alpha$ $\in$ $Head(\sigma)$ mentions none
			or all of the variables in $Body(\sigma)$ \cite{BagetLMS11}.
		\item[](``LR $\subsetneq$ GLR"):
		    This follows from the fact that the loop pattern Type I of Definition \ref{def-generalised-restricted}
		    is actually the loop pattern of Definition \ref{def-restricted}.
		\item[](``AC $\subsetneq$ GLR"): 
		    On the contrary, assume that there exists some 
			$\Sigma$ $\in$ AC such that $\Sigma$ $\notin$ GLR. Then by Definition \ref{def-generalised-restricted},
			there exists some loop pattern $L$ $=$ $(\alpha_1,\rho_1)\cdots(\alpha_n, \rho_n)$ such that
			it is neither of the Types I-V as described in Definition \ref{def-generalised-restricted}.
			In particular, we have that $L$ is not of the Type II. Then this implies that for all 
			$(\alpha_i,\rho_i)$ ($1< i\leq n$), we have that $\mathsf{body}(\rho_i)$ separated 
			into two disjoint body parts $\mathsf{body}(\rho_i)$ $=$ $\mathsf{body_h}(\rho_i)$ $\cup$ 
			$\mathsf{body_b}(\rho_i)$ implies that for all $j$ ($1\leq j < i$), one of the following 
			conditions holds:	
		    \begin{enumerate}
		        \item $\mathsf{body_h}(\rho_i)\cap \mathsf{body_b}(\rho_i)\neq\emptyset$, or
		        \item $\mathsf{var}\big(\{\alpha_j\}\cup \mathsf{body_h}(\rho_i)\big)$
		              $\cap$ $\mathsf{var}\big(\mathsf{body_b}(\rho_i)\big)$ $\neq$
		              $\emptyset$.		          
		    \end{enumerate}	
			In particular, if we take $\mathsf{body_h}(\rho_i)$ $=$ $\emptyset$ and 
			$\mathsf{body_b}(\rho_i)$ $=$ $\mathsf{body}(\rho_i)$, for each $i$ $\in$ $\{1,\ldots,n\}$,
			then since $L$ is a loop pattern 
			(and thus, $\alpha_{i+1}$ $\in$ $\mathsf{body}(\rho_i)$ $=$ $\mathsf{body_b}(\rho_i)$)
			then we have that Conditions 1 and 2 cannot hold. Therefore, we must have that 
			Condition 3 holds for each $(\alpha_i,\rho_i)$ ($1\leq i <n$) (i.e., if we take
			$\mathsf{body_h}(\rho_i)$ $=$ $\emptyset$ and $\mathsf{body_b}(\rho_i)$ $=$ 
			$\mathsf{body}(\rho_i)$). Then this contradicts the assumption that $\Sigma$ $\in$ aGRD
			because this implies a cycle in the ``firing graph" \cite{Baget04}.	
		\item[](``ML $\subsetneq$ GLR"):
			On the contrary, assume that there exists some 
			$\Sigma$ $\in$ ML such that $\Sigma$ $\notin$ GLR. Then again by Definition 
			\ref{def-generalised-restricted}, there exists some loop pattern $L$ $=$ 
			$(\alpha_1,\rho_1)\cdots(\alpha_n, \rho_n)$ such that it is neither of the Types I-V 
			as described in Definition \ref{def-generalised-restricted}.
			In particular, we have that $L$ is not of the Type III. Then this implies that
			there exists some $(\alpha_i,\rho_i)$ ($1\leq i <n$) such that 	
			${\sf var}(\rho_i)$ $\not\subseteq$ ${\sf var}(\beta)$, for some 
			$\beta$ $\in$ ${\sf body}(\rho_i)$. Therefore, since $\rho_i$ $=$ $\sigma_i\theta_i$, 
			for some $\sigma_i$ $\in$ $\Sigma$ and assignment $\theta_i$, then it follows that
			there exists some $\beta'$ $\in$ ${\sf body}(\sigma_i)$ such that
			${\sf var}(\sigma_i)$ $\not\subseteq$ ${\sf var}(\beta')$. Then this contradicts the 
			assumption that $\Sigma$ $\in$ ML.				
		\item[](``SJ $\subsetneq$ GLR"):
			On the contrary, assume that there exists some 
			$\Sigma$ $\in$ SJ such that $\Sigma$ $\notin$ GLR. Then again by Definition 
			\ref{def-generalised-restricted}, there exists some loop pattern $L$ $=$ 
			$(\alpha_1,\rho_1)\cdots(\alpha_n, \rho_n)$ such that it is neither of the Types I-V 
			as described in Definition \ref{def-generalised-restricted}. 
			In particular, we have that $L$ is not of the Type IV. Then this implies that there 
			exists some pair $(\alpha_i,\rho_i)$ in $L$ ($1\leq i< n$) such that 
			$\big({\sf var}(\alpha_i)$ $\cap$ ${\sf var}(\beta)\big)$ 
			$\not\subseteq$ $\bigcap_{j=i+1}^n{\sf var}(\alpha_j)$,
			for some $\beta$ $\in$ ${\sf body}(\rho_{i+1})\setminus\{\alpha_i\}$. Then this again
			contradicts the assumption that $\Sigma$ $\in$ SJ since the ``expansion" of $\Sigma$ 
			\cite{CaliGP12} (which correspond to the loop pattern) will contain a marked variable 
			that occurs in two different atoms;
			
		\item[](``aGRD $\subsetneq$ GLR"):
			On the contrary, assume that there exists some 
			$\Sigma$ $\in$ aGRD and $\Sigma$ $\notin$ GLR. Then again by Definition 
			\ref{def-generalised-restricted}, there exists some loop pattern $L$ $=$ 
			$(\alpha_1,\rho_1)\cdots(\alpha_n, \rho_n)$ such that it is neither of the Types I-V 
			as described in Definition \ref{def-generalised-restricted}. 
			In particular, we have that $L$ is not of the Type V. Then we have from the definition
			of loop restricted Type V that $\Sigma$ will have cycle in the rule dependency graph,
			which contradicts the assumption that $\Sigma$ $\in$ aGRD.  
	\end{description}

	\vspace*{-0.5cm}
	
\end{proof}

In \cite{wr-2012}, the weakly recursive (WR) class of simple TGDs was proposed.
A set of TGDs $\Sigma$ is \textit{simple} if for each $\sigma$ $\in$ $\Sigma$, 
each atom $\alpha$ in $\sigma$ does not have any occurrence of constants and repeated variables.
For a given set $\Sigma$ of simple TGDs, Civili and Rosati considered $\Sigma$'s {\em position graph},
and defined $\Sigma$ to be {\em weakly recursive} if $\Sigma$'s position graph does not contain
any cycles that have edges with explicit or implicit variable transitive connections. The detailed definition of
WR class of simple TGDs is referred to \cite{wr-2012}.

It was then shown in \cite{wr-2012} that the WR class captures all existing known first-order rewritable classes
when restricted to simple TGDs. The following result shows that in the case of simple TGDs, 
WR and GLR are two incomparable first-order rewritable classes. 

\begin{proposition}\label{WR_captures_GLR}
	Under the restriction to simple TGDs, we have that \emph{GLR} $\not\subseteq$ \emph{WR} and
	\emph{WR} $\not\subseteq$ \emph{GLR}.	
\end{proposition}
\comment{
\begin{proof}
	(``$\not\subseteq$") From the proof of Theorem 5 in \cite{wr-2012}, 
	we consider a set $\Sigma$ of simple TGDs comprising of the following two rules:
	\begin{align}
		&{\sf s}(X,Y,Z,V)\ra{\sf r}(X,Y,Z),\label{tgd_excl_proof_1}\\
		&{\sf t}(X,W)\wedge{\sf r}(X,W,Y)\ra\exists Z\,{\sf s}(X,Y,Z,W).\label{tgd_excl_proof_2}
	\end{align}
	Then we can show that $\Sigma$ is not in the GLR class of simple TGDs.  

    (``$\subseteq$")
	Assume on the contrary that for some set of simple TGDs $\Sigma$, we have that 
	$\Sigma$ $\in$ GLR but $\Sigma$ $\notin$ WR. Then from the Definitions 1 and 2 in 
	\cite{wr-2012}, it follows that the \emph{position graph} \cite{wr-2012} of $\Sigma$
	has a cycle that contains both an $m$-edge and an $s$-edge. Then since $R$-compatibility
	corresponds to derivation paths because the TGDs considered are simple, then we can
	construct a loop pattern $L$ $=$ $(\alpha_1,\rho_1)\ldots(\alpha_n,\rho_n)$ that 
	corresponds to the position graph of $\Sigma$. Therefore, the fact that the $s$-edge
	corresponds to ``existential" body variable 
	(i.e., variable mentioned in ${\sf body}(\sigma)$ but not in ${\sf head}(\sigma)$)
	that is shared by two body atoms, then we have that $L$ does not fall into any
	of the loop pattern Types I-IV in definition \ref{def-generalised-restricted}.  
\end{proof}
}
\begin{proof}
	(``GLR $\not\subseteq$ WR") From the proof of Theorem 5 in \cite{wr-2012}, 
	we consider a set $\Sigma$ of simple TGDs comprising of the following two rules:
	\begin{align}
		&{\sf s}(X,Y,Z,V)\ra{\sf r}(X,Y,Z),\label{tgd_excl_proof_1}\\
		&{\sf t}(X,W)\wedge{\sf r}(X,W,Y)\ra\exists Z\,{\sf s}(X,Y,Z,W).\label{tgd_excl_proof_2}
	\end{align}
	Then we get that $\Sigma$ is not in the GLR class of simple TGDs.  

	
	
    (``WR $\not\subseteq$ GLR")	Consider the following set of TGDs $\Sigma'$: 
	\begin{align}
		&{\sf r}(X,Y)\wedge{\sf r}(Y,Z)\ra\exists\,U{ \sf s}(X,Z,U),\label{tgd_proof_1}\\
		&{\sf s}(X,Z,U)\wedge{\sf t}(X,U)\ra {\sf t}(Z,U),\label{tgd_proof_2}\\
		&{\sf t}(X,U)\wedge{\sf t}(Z,U)\ra {\sf r}(X,Z).\label{tgd_proof_3}
	\end{align}
	Then it can be checked that $\Sigma'$ is not WR because 
	we will have a cycle $\LA r[],t[]\RA$, $\LA t[],s[]\RA$, $\LA s[],r[]\RA$ 
	in the ``position graph" \cite{wr-2012} of $\Sigma'$ and where the edge $\LA s[],r[]\RA$ will
	have both an $m$ and $s$ label. On the other hand, we have that $\Sigma'$ is aGRD, which is
	also GLR by Proposition \ref{glr-contains}.  
\end{proof}

\comment{
Although the WR class captures our GLR class under the restriction to simple TGDs, we should highlight that
our results presented in this paper are for arbitrary TGDs. 
It seems that simple TGDs are rather restricted for representing general knowledge domains. 
For instance, considering Example \ref{ex1} discussed in Introduction,
the set $\Sigma_{\mathsf{Research}}$ of TGDs is loop restricted but not simple, 
as atom $\mathsf{projDep}(X,Y,Y)$ in $\sigma_2$ mentions variable $Y$ more than once.
}

We emphasize that
our results presented in this paper are for arbitrary TGDs, while
 simple TGDs are probably restricted for representing general knowledge domains.

\section{Concluding Remarks}

Loops have been an important concept in the study for traditional Datalog
programs, and then have been employed and extended in Answer Set Programming research in recent years, 
e.g., \cite{ChenLZZ11,lin04,yan-10,aij-zz17}. In this paper, through a series of novel definitions of 
derivation paths, derivation trees and loop patterns, we are able 
to discover new decidable classes of TGDs for ontology based query answering 
using a very different idea from previous approaches.

%

As we have showed, the class of GLR TGDs properly
contains all other first-order rewritable TGDs classes for general TGDs.
We believe that our results presented in this paper will be useful in developing efficient OBDA systems for 
broader application domains.

%


\bibliographystyle{aaai}
\bibliography{aaai18-New}

\begin{thebibliography}{}

\bibitem[\protect\citeauthoryear{Baader \bgroup et al\mbox.\egroup
  }{2016}]{baader:jair16}
Baader, F.; Bienvenu, M.; Lutz, C.; and Wolter, F.
\newblock 2016.
\newblock Query and predicate emptiness in ontology-based data access.
\newblock {\em Journal of Artificial Intelligence Research} 56:1--59.

\bibitem[\protect\citeauthoryear{Baget \bgroup et al\mbox.\egroup
  }{2011}]{BagetLMS11}
Baget, J.; Lecl{\`{e}}re, M.; Mugnier, M.; and Salvat, E.
\newblock 2011.
\newblock On rules with existential variables: Walking the decidability line.
\newblock {\em Artifificial Intelligence} 175(9-10):1620--1654.

\bibitem[\protect\citeauthoryear{Baget}{2004}]{Baget04}
Baget, J.
\newblock 2004.
\newblock Improving the forward chaining algorithm for conceptual graphs rules.
\newblock In {\em Principles of Knowledge Representation and Reasoning:
  Proceedings of the Ninth International Conference (KR2004), Whistler, Canada,
  June 2-5, 2004},  407--414.

\bibitem[\protect\citeauthoryear{Bienvenu}{2016}]{bien:ijcai16}
Bienvenu, M.
\newblock 2016.
\newblock Ontology-mediated query answering: Harnessing knowledge to get more
  from data.
\newblock In {\em Proceedings of IJCAI 2016},  4058--4061.

\bibitem[\protect\citeauthoryear{Cal{\`{\i}}, Gottlob, and
  Kifer}{2008}]{CaliGK08}
Cal{\`{\i}}, A.; Gottlob, G.; and Kifer, M.
\newblock 2008.
\newblock Taming the infinite chase: Query answering under expressive
  relational constraints.
\newblock In {\em Proceedings of the 21st International Workshop on Description
  Logics (DL2008), Dresden, Germany, May 13-16, 2008}.

\bibitem[\protect\citeauthoryear{Cal{\`{\i}}, Gottlob, and
  Lukasiewicz}{2012}]{CaliGL12}
Cal{\`{\i}}, A.; Gottlob, G.; and Lukasiewicz, T.
\newblock 2012.
\newblock A general datalog-based framework for tractable query answering over
  ontologies.
\newblock {\em J. Web Sem.} 14:57--83.

\bibitem[\protect\citeauthoryear{Cal{\`{\i}}, Gottlob, and
  Pieris}{2012}]{CaliGP12}
Cal{\`{\i}}, A.; Gottlob, G.; and Pieris, A.
\newblock 2012.
\newblock Towards more expressive ontology languages: The query answering
  problem.
\newblock {\em Artif. Intell.} 193:87--128.

\bibitem[\protect\citeauthoryear{Chen \bgroup et al\mbox.\egroup
  }{2011}]{ChenLZZ11}
Chen, Y.; Lin, F.; Zhang, Y.; and Zhou, Y.
\newblock 2011.
\newblock Loop-separable programs and their first-order definability.
\newblock {\em Artificial Intelligence} 175(3-4):890--913.

\bibitem[\protect\citeauthoryear{Civili and Rosati}{2012}]{wr-2012}
Civili, C., and Rosati, R.
\newblock 2012.
\newblock A broad class of first-order rewritable tuple-generating
  dependencies.
\newblock In {\em Proceedings of the 2nd International Conference on Datalog in
  Academia and Industry (Datalog-2012)},  68--80.

\bibitem[\protect\citeauthoryear{Deutsch, Nash, and Remmel}{2008}]{DeutschNR08}
Deutsch, A.; Nash, A.; and Remmel, J.~B.
\newblock 2008.
\newblock The chase revisited.
\newblock In {\em Proceedings of the Twenty-Seventh {ACM}
  {SIGMOD-SIGACT-SIGART} Symposium on Principles of Database Systems, {PODS}
  2008, June 9-11, 2008, Vancouver, BC, Canada},  149--158.

\bibitem[\protect\citeauthoryear{Eiter, Lukasiewicz, and
  Predoiu}{2016}]{eiter:kr16}
Eiter, T.; Lukasiewicz, T.; and Predoiu, L.
\newblock 2016.
\newblock Generalized consistent query answering under existential rules.
\newblock In {\em Proceedings of KR 2016},  359--368.

\bibitem[\protect\citeauthoryear{Fagin \bgroup et al\mbox.\egroup
  }{2005}]{FaginKMP05}
Fagin, R.; Kolaitis, P.~G.; Miller, R.~J.; and Popa, L.
\newblock 2005.
\newblock Data exchange: semantics and query answering.
\newblock {\em Theor. Comput. Sci.} 336(1):89--124.

\bibitem[\protect\citeauthoryear{Gottlob, Manna, and
  Pieris}{2013}]{GottlobMP13}
Gottlob, G.; Manna, M.; and Pieris, A.
\newblock 2013.
\newblock Combining decidability paradigms for existential rules.
\newblock {\em Theory and Practice of Logic Programming} 16(1):877--892.

\bibitem[\protect\citeauthoryear{Grau \bgroup et al\mbox.\egroup }{2013}]{g13}
Grau, B.~C.; Horrocks, I.; Krotzsch, M.; Kupke, C.; Magka, D.; Motik, B.; and
  Wang, Z.
\newblock 2013.
\newblock acyclicity notions for existential rules and their application to
  rqquery raanswering in ontologies.
\newblock {\em Journal of Artificial Intelligence Research} 47:741--808.

\bibitem[\protect\citeauthoryear{Hansen \bgroup et al\mbox.\egroup
  }{2015}]{hansen:ijcai15}
Hansen, P.; Lutz, C.; Seylan, I.~Ì.; and Wolter, F.
\newblock 2015.
\newblock Efficient query rewriting in the description logic el and beyond.
\newblock In {\em Proceedings of IJCAI 2015},  3034--3040.

\bibitem[\protect\citeauthoryear{Kaminski, Nenov, and Grau}{2014}]{Kam:2014}
Kaminski, M.; Nenov, Y.; and Grau, B.~C.
\newblock 2014.
\newblock Computing datalog rewritings for disjunctive datalog programs and
  description logic ontologies.
\newblock In {\em Web Reasoning 2014},  76--91.

\bibitem[\protect\citeauthoryear{Kontchakov, Rodriguez-Muro, and
  Zakharyaschev}{2013}]{kon:2013}
Kontchakov, R.; Rodriguez-Muro, M.; and Zakharyaschev, M.
\newblock 2013.
\newblock Ontology-based data access with databases: A short course.
\newblock In {\em Reasoning Web 2013},  194–229.

\bibitem[\protect\citeauthoryear{Kr{\"{o}}tzsch and
  Rudolph}{2011}]{KrotzschR11}
Kr{\"{o}}tzsch, M., and Rudolph, S.
\newblock 2011.
\newblock Extending decidable existential rules by joining acyclicity and
  guardedness.
\newblock In {\em Proceedings of IJCAI 2011},  963--968.

\bibitem[\protect\citeauthoryear{Leone \bgroup et al\mbox.\egroup
  }{2012}]{LeoneMTV12}
Leone, N.; Manna, M.; Terracina, G.; and Veltri, P.
\newblock 2012.
\newblock Efficiently computable datalog{\(\exists\)} programs.
\newblock In {\em Principles of Knowledge Representation and Reasoning:
  Proceedings of the Thirteenth International Conference, {KR} 2012, Rome,
  Italy, June 10-14, 2012}.

\bibitem[\protect\citeauthoryear{Lin and Zhou}{2004}]{lin04}
Lin, F., and Zhou, Y.
\newblock 2004.
\newblock Assat: Computing answer sets of a logic program by sat solvers.
\newblock {\em Artificial Intelligence} 157:115--137.

\bibitem[\protect\citeauthoryear{Nikolaou \bgroup et al\mbox.\egroup
  }{2017}]{niko:2017}
Nikolaou, C.; Kostylev, E.~V.; Konstantinidis, G.; Kaminski, M.; Grau, B.~C.;
  and Horrocks, I.
\newblock 2017.
\newblock The bag semantics of ontology-based data access.
\newblock In {\em https://arxiv.org/abs/1705.07105}.

\bibitem[\protect\citeauthoryear{Patel{-}Schneider and
  Horrocks}{2007}]{Patel-SchneiderH07}
Patel{-}Schneider, P.~F., and Horrocks, I.
\newblock 2007.
\newblock A comparison of two modelling paradigms in the semantic web.
\newblock {\em Journal of Web Semantics} 5(4):240--250.

\bibitem[\protect\citeauthoryear{Zhang and Zhou}{2010}]{yan-10}
Zhang, Y., and Zhou, Y.
\newblock 2010.
\newblock On the progression semantics and boundedness of answer set programs.
\newblock In {\em Proceedings of the 12th International Conference on the
  Principles of Knowledge Representation and Reasoning (KR-2010)},  518--527.

\bibitem[\protect\citeauthoryear{Zhou and Zhang}{2017}]{aij-zz17}
Zhou, Y., and Zhang, Y.
\newblock 2017.
\newblock A pregression semantics for first-order logic programs.
\newblock {\em Artificial Intelligence} to appear.

\end{thebibliography}

\end{document}